\documentclass{article}
\usepackage{arxiv}
\usepackage[T1]{fontenc}
\usepackage{float}
\usepackage{url}
\usepackage[hidelinks]{hyperref}
\usepackage[utf8]{inputenc}
\usepackage{graphicx}
\usepackage{amssymb}
\usepackage{booktabs}
\usepackage{algorithm}
\usepackage{algorithmic}
\usepackage{tikz}
\usepackage{amsmath}
\usepackage{amsthm}
\usepackage{cancel}
\usepackage{color}
\usepackage{comment}
\usepackage{mathtools, stmaryrd}
\usepackage[compatibility=false]{caption}
\usepackage{subcaption}
\usepackage{color}
\usepackage{lineno}

\newtheorem{Definition}{Definition}
\newtheorem{Proposition}{Proposition}
\newtheorem{Example}{Example}
\newtheorem{Prop}{Property}
\newtheorem{CE}{Counterexample}
\newtheorem{Ob}{Observation}
\newtheorem{Rem}{Remark}

\title{From numerical proportions to analogical proportions between probabilities}

\author{ 
Henri Prade \\
	IRIT \\ Cr Rose Dieng-Kuntz \\
    Toulouse, 31400, France\\
	\texttt{henri.prade@irit.fr} 
	\And
Gilles Richard \\
	IRIT \\ Cr Rose Dieng-Kuntz \\
    Toulouse, 31400, France\\
	\texttt{gilles.richard@irit.fr}
}

\date{}
\renewcommand{\shorttitle}{From numerical proportions to analogical proportions between probabilities}

\hypersetup{
pdftitle={From numerical proportions to analogical proportions between probabilities},
pdfsubject={q-bio.NC, q-bio.QM},
pdfauthor={Henri Prade, Gilles Richard},
pdfkeywords={analogical proportion, probability, numerical proportion},
}

\begin{document}
\maketitle
\begin{abstract}
Analogical proportions link four items $a, b, c, d$  by a relation stating that ``$a$ is to $b$ as $c$ is to $d$",  $a, b, c, d$ being the formal representation of real world entities, ranging from simple numerical values to more complex structures such as profiles. Accordingly, $a, b, c, d$ could be atomic values
like Boolean, nominal or numerical  values,
more generally vectors of such values, or even families of items represented by
logical formulas. In this paper, we consider another representation setting, which is the probabilistic one. 
Precisely, the article proposes a study of {analogical} proportions between probabilities, whether they are simply between probability values, or between distributions (which requires the preservation of their normalization). More particularly, we study the properties of definitions based on arithmetic proportion, or on a combination of the former with geometric proportion, while other options are also discussed. 
Previous works have shown that when four profiles $a, b, c, d$, represented as vectors, form analogical proportions componentwise, it is likely that their classes form an  analogical proportion also. This is the basis of an analogical proportion-based classification method that can produce   accurate predictions. Similarly, in this paper  each profile is associated with a distribution  describing the frequencies of the possible values of a discrete attribute of interest. We then {discuss} and experimentally investigate  if the distributions {associated to}
four profiles forming an analogical proportion {themselves} also form an analogical proportion.
\end{abstract}
\section{Introduction}\label{intro}
Analogies and probabilities are rarely  considered together, although they can be both related to induction \cite{PraRicJAL2024}. However, already in the 1930s, Janina Hosiasson-Lindenbaum \cite{Sznajder}, a formal philosopher, specialist of probabilistic reasoning and logician, considered the following analogical reasoning: if two conjectures $f_1$ and $f_2$ are consequences of the same hypothesis $h$, the observation of $f_1$ (which therefore becomes a fact) induces $h$ and increases belief in conjecture $f_2$. 
She sought additional hypotheses allowing her a probabilistic justification for this increase in belief. 
Later, in the same spirit, Polya \cite{Polya1954} will consider that $f_1$ and $f_2$ are analogous if they are implied by a common hypothesis.

In this article we are interested in another approach where analogies are expressed in the form of analogical proportions.\footnote {Polya's view of analogy and the analogical proportion-based formalization are quite different. However, if $a$ and $b$ on the one hand, and $c$ and $d$ on the other hand are ``analogous'', one can argue that an analogical proportion ``$a$ is to $b$ as $c$ is to $d$'' must hold. This is shown in \cite{ijis/PradeR11} where ``$a$ analogous to $b$'' is modeled, quite differently from Polya \cite{Polya1954}, by the non-monotonic consequence relations ``if $a$ generally $b$'' and ``if $b$ generally $a$''.} 
That is, statements of the form
``$a$ is to $b$ as $c$ is to $d$'' that relate four items. 
These items can refer to a wide range of possibilities, such as pieces of information, words, phrases, pieces of text (as explored in \cite{AfaLimPraRicIJCAI2022} or more recently in the context of LLM \cite{Webb2023,MitchellAnalogyLLM2024}), or even images \cite{PraRicIJCAI2021}. 
This approach  departs from the ones that
align two situations whose components are put into correspondence on the basis of the similarity of the relations that link them in each situation, as in  the structure mapping theory \cite{Gentner1983,ai/FalkenhainerFG89,GenHolKok2001}.
It also differs from its variants such as
\cite{KLD1994}, including 
 the CWSG (``copy with substitution and generation'') inference algorithm  \cite{Holyoak}, as well as from a 
 framework such as HDTP \cite{HDTPGUST2006} based on second order logic.

Analogical proportions are supposed to obey a few postulates that mimic key properties of numerical proportions. Analogical proportions have a long history that dates back to Ancient Greek time. However, if we except the pioneering work of S. Klein \cite{Klein83},   it is in the last two or three decades that they have encountered a renewal of interest with attempts at  formalizing them \cite{PirYvo99,LepAge2001,DelhayM04,StrYvoTAL2006,MicPraECSQARU2009}. At the formal level, analogical proportions  
{ can both be defined for atomic values like Boolean, nominal or numerical, or for vectors thereof via componentwise extension.} Recent works have  witnessed a renewed interest  for defining analogical proportions between numbers (in different ways) \cite{LepCou2024,sum/PradeR24}. Moreover, analogical proportions can be extended to general representational frameworks such as propositional logic \cite{diagrams/PradeR24}. 
In this article, we consider another representation setting: the probabilistic framework. 

Analogical proportion-based inference, {viewed as a particular case of analogical jump  \cite{DavRus1987},} has been applied to classification  \cite{ijar/BounhasPR17}. In classification, $a$, $b$, $c$, $d$ are vectors of attribute values describing the examples and the items to be classified. The attributes may be Boolean, nominal, or numeric. Like Bayesian probabilistic inference, analogical proportions-based inference allows classification tasks to be performed successfully \cite{MicBayDelJAIR2008,ijar/BounhasP23,ijar/BounhasP24}.  

This raises the question of the  compatibility between analogical proportions and probabilities. 
More precisely,   
{ considering} a quadruplet of vectors describing four items 
{  such that} analogical proportions hold componentwise, 
{  if each item is associated with a probability,}
one may wonder under what condition  an analogical proportion may also hold   between these probabilities. 
The question extends to probability distributions.

This article\footnote{The paper, for most of its theoretical part (with the exceptions of Properties \ref{linea}, \ref{Prod}, \ref{arigeoind}, Observation \ref{event},  Proposition \ref{TV} and subsection \ref{ad-extension}), results of the reorganization of the main parts of the contents of two conference papers \cite{PraRicEcsqaru2025} \cite{miwai/PradeR25}.
A preliminary version of the first paper, in French,  was presented in a national workshop \cite{PraRicJIAF2025}. The Section  \ref{compat} has been thoroughly revised and the experiments reported in Section \ref{illustration} are entirely new.} is divided into six main sections and two annexes. 
Section \ref{background} restates the necessary background on analogical proportions between Boolean or nominal attribute values. It also discusses analogical proportions between Boolean formulas. and recalls their link with weak multivalued dependencies. Section \ref{numerical} is devoted to 
 numerical analogical proportions, emphasizing the parallel with the Boolean and the nominal case. Numerical proportions can be classically considered either in terms of arithmetic proportions or in terms of geometric proportions, which are first considered in subsection  \ref{basic}. Then existing proposals for numerical analogical proportions are reviewed in subsection \ref{existing}. General requirements for 
 a definition of numerical analogical proportions are presented in subsection \ref{general}. Lastly the interest of a definition combining arithmetic and  geometric proportions {is investigated} in subsection \ref{ari+geo}. 
 
 Section \ref{proba} deals with analogical proportions between probabilities. Subsection \ref{probaval} focuses on analogical proportions between probability values: these values represent the probabilities that an item has particular attribute values. 
 The following subsections explore analogical proportions between four probability distributions over the same finite set.
 {Two definitions are specifically examined}: the first definition is based on the arithmetic proportion only since using only the geometric proportion presents difficulties (subsection \ref{probadis}). 
 This definition  preserves the total variation distance between the distributions; moreover it is used for computing an analogical dissimilarity between four distributions by {quantifying how far they are from forming} an analogical proportion (subsection \ref{ad-extension}). 
 The second definition requires both arithmetic and geometric proportions to be satisfied componentwise (subsection \ref{arigeodist}). This second definition is consistent with the preservation of the Kullback-Leibler divergence between pairs of distributions. Subsection \ref{further} briefly considers the case of continuous probability density functions on the real line, and an  appendix exhibits examples of analogical proportions between probabilities densities.

Section \ref{compat} discusses how an analogical proportion between four profiles
may be associated with an analogical proportion between probabilities. First, we recall how analogical proportions between profiles can be associated with analogical proportions between classes (subsection \ref{classification}). Then we describe a situation where profiles forming
an analogical proportion can  be naturally
associated with probabilities {that are not, in general, themselves}
in analogical proportion unless particular adjustment takes place (subsection \ref{compa}). 
Finally, we present a pattern of analogical inference where profiles forming
an analogical proportion  are associated with probability distributions that may be expected to {also} form an analogical proportion (subsection \ref{ana+distri}). 
Section \ref{illustration} { empirically validates this expectation on}  
two large datasets.
The section ends with a final discussion proposing lines for further research.

\section{Boolean and nominal analogical proportions}\label{background}
An analogical proportion  
is a relation between any four items $a, b, c, d$, which is denoted $a:b::c:d$, and reads ``$a$ is to $b$ as $c$ is to  $d$''. Following a parallel (which dates back to Aristotle) with the  arithmetic ($a - b = c - d$) and the geometric proportions ($\frac{a}{b}=\frac{c}{d} $)  where $a, b, c, d$ are real numbers, analogical proportions  
are supposed to obey the three
core postulates\footnote{ See \cite{AfaLimPraRicIJCAI2022} for a discussion of analogical proportions in natural languages.}: 
\begin{itemize}
   \item \emph{reflexivity}: $a:b::a:b$;
    \item \emph{symmetry}: $a:b::c:d \Rightarrow c:d::a:b$ ;
    \item {stability under \emph{central permutation}}: $a:b::c:d \Rightarrow a:c::b:d$. 
 \end{itemize}
As a consequence, the repeated and alternate application of symmetry and central permutation shows that if $a:b::c:d$ holds, 7 other patterns $c:d::a:b$, $c:a::d:b$, $d:b::c:a$, $d:c::b:a$, $b:a::d:c$, $b:d::a:c$, $a:c::b:d$ hold as well, among the 24 possible orderings of the four items.

Thus, an analogical proportion 
also satisfies \begin{itemize}
    \item   $a:b::c:d$  $\Rightarrow d:b::c:a$ ({\it external permutation}); 
  \item $a:b::c:d \Rightarrow b : a :: d : c$ ({\it internal reversal});
  \item $a:b::c:d \Rightarrow d : c :: b : a$ ({\it complete reversal}).
\end{itemize}
In this paper $a, b, c, d$ are Boolean, categorical, or numerical values, or vectors thereof. When the items are represented by vectors of attribute values  $a=(a_1, \ldots, a_n)$, $b= (b_1, \ldots, b_n)$, $c= (c_1, \ldots, c_n) $, and $d= (d_1, \ldots, d_n)$, the analogical proportion $a:b::c:d$ is defined componentwise:
\begin{Definition}\label{DefVec} $a : b :: c : d $ holds 
if  $\forall i \in [1,n], \  a_i:b_i::c_i:d_i$ holds.
\end{Definition}
In the next two subsections, we successively consider cases where the components $i$ are Boolean and nominal. 
We remove the $i$ subscripts to simplify the notation, since we  only focus on a single component in the next two subsections.

\subsection{Boolean model}\label{Boo}
The reflexivity and the central permutation postulates force  $a:a::b:b$ to be considered as a valid proportion. Thus when $a, b, c, d$ are Boolean variables (i.e. taking values in $\mathbb{B}=\{0, 1\}$), the minimal Boolean model \cite{PraRicIJAR2018} of the three previous postulates is given in Table \ref{truthTableAnalogy}, which exhibits the 6 valuations of $a,b,c,d$ that any model of the 3 postulates must include. Note that symmetry is obtained for free in this table.
{ This table also exhibits the other 10 (= $2^4 - 6$) valuations which are not valid in the minimal model.} 
A logical expression that is true only for the {valid} quadruplets in Table \ref{truthTableAnalogy} is given by the formula \cite{MicPraECSQARU2009}:

\begin{equation} \label{formuledif}
a : b :: c : d=[(a \wedge \neg b)  \equiv (c \wedge \neg d)] \wedge [(\neg a \wedge b)  \equiv (\neg c \wedge d)]    
\end{equation}    

This formula expresses precisely that ``$a$ {differs} from {$b$} as {$c$} {differs} from {$d$}, and {$b$} {differs} from {$a$} as {$d$} {differs} from {$c$}'' (lines 5 and 6  in Table \ref{truthTableAnalogy}) and “when {$a$} and {$b$} do not differ, {$c$} and {$d$} do not differ” (lines 1, 2 and lines 3, 4 in Table \ref{truthTableAnalogy}).
\begin{table}[!ht]
\caption{Valid or invalid Boolean analogical proportions with their $AD$}
\label{truthTableAnalogy}
\centering
$
\begin{array}{c|cccc|c||c|cccc|c}
\multicolumn{6}{c}{\bf Valid}  &  \multicolumn{6}{c}{\bf Invalid} \\
& a & b & c & d & AD & & a & b & c & d  & AD\\
\hline
\hline
1. &0 & 0 & 0 & 0 & 0 & 7. & 0 & 0 & 0 & 1  &  1\\
\hline
2. & 1 & 1 & 1 & 1 & 0 & 8. & 0 & 0 & 1 & 0  &  1\\
\hline
3. & 0 & 0 & 1 & 1 & 0 & 9. & 0 & 1 & 0 & 0  &  1\\
\hline
4. &1 & 1 & 0 & 0 & 0 & 10. & 1 & 0 & 0 & 0 &  1\\
\hline
5. & 0 & 1 & 0 & 1 & 0& 11. & 1 & 1 & 1 & 0 &  1\\
\hline
6. & 1 & 0 & 1 & 0 & 0 & 12. & 1 & 1 & 0 & 1 &  1\\
\hline
& & & & & &  13. &1 & 0 & 1 & 1 &  1\\
\hline
& & &  & & & 14. & 0 & 1 & 1 & 1  &  1\\
\hline
& & & & & & 15. & 1 & 0 & 0 & 1 &  2\\
\hline
& & & & & & 16. & 0 & 1 & 1 & 0 &  2\\
\hline
\end{array}
$
\end{table}
Boolean analogical proportions not only satisfy the three postulates, but also enjoy remarkable properties that are not consequences of the postulates:

- {\it code independence} \cite{PraRicLU2013}: $a : b :: c : d \   \Rightarrow \neg a : \neg b :: \neg c: \neg d $;

- {\it transitivity} \cite{PraRicLU2013}: $a:b::c:d, c:d::e:f \Rightarrow a:b::e:f$;

- {\it C-transitivity} \cite{BarbotMPAIJ19}\cite{ijcai/SchockaertIG21}: $a:b::c:d, a:b'::c':d \Rightarrow b:b'::c':c$.\footnote{This is easy to check. If $a = d = 1$ (resp. 0) then $b= c=b'=c' = 1$ (resp. 0) and then $b:b'::c':c$ holds. If $(a, d) = (0, 1)$ then $(b, c) = (0, 1)$ or $(b, c) = (1, 0)$, and $(b', c') = (0, 1)$ or $(b, c) = (1, 0)$, which yields 4 possible cases where it can be checked that $b:b'::c':c$ holds. The case  $(a, d) = (1, 0)$ is similar.} 

Note that, in agreement with internal reversal property, C-transitivity also yields $b':b::c:c'$.\\ 

Analogical inference is based on finding $d$ given $a$, $b$, $c$ {and relies on the minimal model}.
Table \ref{truthTableAnalogy} shows that if it exists, $d$ is unique, but that there does not always exist $x$ such that $a : b :: c : x$ is true.
{Indeed,} there is no solution in $ \{0, 1\}$ for $1 : 0 :: 0 : x$ and $0 : 1 :: 1 : x$.

The equation $a : x :: y : d$ true has 3 kinds of solution depending on the values of $a$ and $d$; it  corresponds to 3 types of patterns in  Table \ref{truthTableAnalogy}:

- if $a = d$ then $x= y= a = d$ (lines 1 and 2 in  Table \ref{truthTableAnalogy}) ;
 
- if $a \neq d$ then $x = a$ and $y=d$ (lines 3 and 4 in  Table \ref{truthTableAnalogy}), 

- \quad\quad\quad\quad\ \ or $x = d$ and $y = a$ (lines 5 and 6 in  Table \ref{truthTableAnalogy}). \\

Besides, the concept of {\it analogical dissimilarity}, introduced by Miclet {\it et al.} in \cite{MicBayDelJAIR2008},  arose from an attempt to formalize a relaxed version of analogy, capturing linguistic expressions of the form ``$a$ is to $b$ {\it almost} as $c$ is to $d$''. To remain consistent with the notion of analogical proportions, the authors proposed quantifying the term ``almost'' with a positive real value $AD(a, b, c, d)$. This value is equal to 0 when the analogy holds exactly and increases as the four objects deviate from a true analogical relationship. In the Boolean case, it corresponds to ``the minimal number of bits that have to be switched in order to produce an analogical proportion''.

A common definition for $AD$, which  applies to Boolean items, as well as to numerical items as we shall see, is $AD(a,b,c,d)=|(b-a)-(d-c)|$.
 In Table \ref{truthTableAnalogy},  the columns $AD$, where the values  vary  between $0$ and $2$, give the analogical dissimilarity for  each 4-tuple. $AD(a, b, c, d)=0$  when the analogical proportion is true, $AD(a, b, c, d)= 1$ for the 8 patterns with odd numbers of 0 and  1 (lines 7 to 14); and AD is maximal when the changes from $a$ to $b$, and from $c$ to $d$ are in opposite directions, i.e., 
 $AD(0110) = AD(1001) = 2$ (lines 15 and 16).

\subsection{Nominal model}\label{nominal}
As can be seen in Table \ref{truthTableAnalogy}, the patterns that make the analogical proportion true are of three types,  namely $ssss$, $sstt$, and $stst$, with $s,t \in \mathbb{B}=\{0, 1\}$, the last two motifs being exchanged by a central permutation. 
 It is now easy to extend to a nominal attribute $\mathcal{A}$ taking its values in a finite domain $\mathcal{D_A}$ whose cardinality can be greater than 2.
Then 
$a : b :: c : d$ is true for a nominal variable associated with an attribute $\mathcal{A}$ if and only if (as first suggested in \cite{PirYvo99}) :
\begin{equation}
(a,b,c,d)\in \{(s,s,s,s),(s,t,s,t),(s,s,t,t)   | s,t \in \mathcal{D_A}\}
\label{nomi}
\end{equation} 
where $s, t$ are  distinct values in {  $\mathcal{D_A}$}.
The condition (\ref{nomi}) generalizes the Boolean case and guarantees the satisfaction of the postulates of analogical proportions.

\subsection{Analogical proportions between Boolean formulas}\label{dia}
Rather than dealing with individual items, each described by a vector of attribute values, one may consider sets of items. 
A logical formula can provide an intensional description of a set of items.
The formula (\ref{formuledif}) of a Boolean analogical proportion also applies to Boolean variables representing any propositional formulas. It states that for the proportion $a:b::c:d$ to be satisfied, the two conditions $(a \wedge \neg b) \equiv (c \wedge \neg d)$ and $(\neg a \wedge b) \equiv (\neg c \wedge d)$ must be satisfied. 

\begin{Example} \label{APLogic} Consider the example of an analogical proportion that appears in \cite{PraRicLU2013} and that is also valid for any lattice structure \cite{BarbotMPAIJ19}:
$$ p \vee q :  p :: q:  p \wedge q$$

It can be easily checked that it holds in the Boolean setting, since $(p \vee q) \wedge \neg p \equiv  q \wedge \neg ( p \wedge q) \equiv  \neg  p\wedge q$ and $ p \wedge \neg ( p \vee q ) \equiv  ( p \wedge q) \wedge \neg q\equiv \bot$. \hfill $\Box$
\end{Example}  

A propositional formula with $n$ variables can be described over the $2^n$ interpretations of the language under consideration.
This is the basis of a bit string representation of logical formulas. For example, if we have two Boolean variables $p$ and $q$, we have the interpretations $pq$, $p\neg q$, $\neg pq$, $\neg p \neg q$. Taking them in this order, $p$ is encoded by $1100$, $p \wedge q$ by $1000$, $p \vee q$ by $1110$, $q$ by $1010$, etc. 

We can show that if  4 logical formulas $a,b,c,d$ involve a set of $n$ variables, and that they are therefore representable by bit strings of size $2^n$, having an analogical proportion  between the 4 logical formulas $a,b,c,d$ amounts to having
an analogical proportion $a_i:b_i::c_i:d_i$ on each component $i$ of the bit strings that represent the formulas, that is to say for each possible interpretation \cite{diagrams/PradeR24} (note that $i$ here refers to an interpretation, and not to a Boolean variable as in the Definition \ref{DefVec}).

\begin{Example} Indeed, in the Example \ref{APLogic}  above, we have 1110 : 1100 :: 1010 : 1000, component by component. \hfill $\Box$
\end{Example}  

Given three propositional formulas, we can find, if it exists, the propositional formula forming an analogical proportion with these three formulas; 
{in the same way we can semantically verify whether four propositional formulas form an analogical proportion; see \cite{diagrams/PradeR24} for details.}

\subsection{Weak multivalued dependency}
Interestingly enough, as shown in \cite{LinkPraRicSUM22}, the notion of analogical proportion coincides with the notion of weak multivalued dependency \cite{FischerG84} in the database field. For the sake of brevity, we just recall the definition and give an example.

Consider a relational table $r$ with a set of attributes $R$, where $X$, $Y$, $Z$ are attributes (or subsets thereof). Then
a weak multivalued dependency \cite{FischerG84} $X \twoheadrightarrow_w Y$ holds in $r$  if, for all tuples $t_1, t_3$  in $r$  
such that $t_1[XY]=t_2[XY]$\footnote{$t_i[XY]$ denotes the restriction of the tuple $t_i$ to attributes in $X\cup Y$.} and $t_1[X(R\setminus Y)]=t_3[X(R\setminus Y)]$ there is some tuple $t_4$ in $r$ such that $t_4[XY]=t_3[XY]$ and $t_4[X(R \setminus Y)]=t_2[X(R\setminus Y)]$. This also expresses a logical independence of $Y$ and $Z$ given $X$.

This is illustrated by the 4 tuples  of the database-like  example in Table \ref{Example} with 3 attributes $X=\{Education\}$, $Y=\{Gender\}$, $Z=\{Age\}$ and domains $\mathcal{D}_{Education}=\{low, medium,  high\}$, $\mathcal{D}_{Gender}=\{F, M\}$,  $\mathcal{D}_{Age}=\{\leq 40,>40\}$. It can be checked that $X \twoheadrightarrow_w Y$ holds, taking   $t_1 = a$, $t_2 = b$, $t_3 = c$, $t_4 = d$, with $XZ = X(R\setminus Y)$. Exchanging $t_2$ and $t_3$ and $Y$ and $Z$ we observe that $X \twoheadrightarrow_w Z$ holds as well in $r$. 
\begin{table}[!ht]
\caption{A database  example}
\label{Example}
\centering
$
\begin{array}{c||c|c|c|c|}

& X = \mbox{\{Education\} } &   Y  =\mbox{ \{Gender\}}  &  Z = \mbox{ \{Age\}}   \\
\hline
a    & \mbox{high} & M     &    \leq 40  \\
\hline
b    & \mbox{high} &   F   &    \leq 40 \\
\hline
c      & \mbox{high} & M     &    > 40 \\
\hline
d & \mbox{high} & F     &   > 40   \\

\hline
\end{array}
$
\end{table}
As can be seen in Table \ref{Example}, $a:b::c:d$ holds and the table presents, in column, component by component, the three types of patterns allowed by the nominal model. Intuitively $a,b,c,d$ represent profiles in a given population and this is where some probability of occurrence in the population could be associated with these 4 profiles, as discussed in Section \ref{compat}.
\section{Numerical analogical proportions}\label{numerical}
In order to accommodate numerical attributes, various proposals for extending analogical proportions to real numbers have been made that will be reviewed, before  proposing a new general definition and presenting some new results.  We first consider the important and classical cases  of arithmetic  and geometric analogical proportions. In this section, $a, b, c, d$, although there is no subscript $i$, refer to {\it single} real numbers. 

\subsection{Arithmetic  and geometric analogical proportions} \label{basic}
Among the numerical proportions, two proportions which involve the four elementary operations have an important place ($a$, $b$, $c$, $d$ are real numbers here):
\begin{itemize}

\item  the arithmetic proportion $ a-b= c-d$, denoted here $a : b ::_{ari} c : d$. 
This clearly fits with the analogical dissimilarity previously introduced at the end of subsection \ref{Boo} 
since $AD(a, b, c, d) = 0  \Longleftrightarrow a : b ::_{ari} c : d$;

\item the geometric proportion $\frac{a}{b}=\frac{c}{d}$,  denoted here $a : b ::_{geo} c : d$.

We will take the convention $\frac{0}{0}=1$ considering that $\lim_{x \to 0} \frac{x}{x}=1$.
\end{itemize}
It is easy to verify that these two proportions
{are indeed analogical proportions since they satisfy the three postulates recalled at the beginning of Section \ref{background}.}
They are also transitive and C-transitive.
As in the Boolean case, they express the identity of the comparison results of $a$ and $b$, and of $c$ and $d$, either in terms of differences or in terms of ratios.
We can also verify the agreement with the truth table of the analogical proportion (Table \ref{truthTableAnalogy}): for the 6 patterns that makes it true, the equalities of the numerical proportions are satisfied (we take $\frac{1}{0}=+\infty$) and there are no equalities for the other possible quadruplets.

Unlike the Boolean case, the number frame offers solutions for continuous analogical proportions of the form $a:m::m:b$. Indeed, they lead to the definition of the arithmetic mean ($m= \frac{a+b}{2}$) and the geometric mean ($m= \sqrt{ab}$).

Because  these proportions satisfy stability under central permutation, one might be tempted to define them by combining  extremes and means like:

-  $ a+d= b+c$  for arithmetic proportion;

- $a\cdot d=b\cdot c$ for geometric proportion.

However, for geometric proportions, this allows us to have  
$0\cdot 1 = 0\cdot 0$, whereas \\$0:0::0:1$ is certainly not
an analogical proportion. We therefore see that the equivalence between the "quotient" and "product" forms excludes $0$ for geometric proportions. {In subsection \ref{general}, we will investigate general conditions for avoiding this issue when combining extremes and means.}

These two proportions are exchanged by a logarithmic/exponential transformation:
   
    if $\frac{a}{b} = \frac{c}{d}$ then we have $\ln(a) - \ln(b)= \ln(c) - \ln(d)$,
    
    if $a -b= c-d$ then we have  $\frac{\mathrm{e}^{a}}{\mathrm{e}^{b}}= \frac{\mathrm{e}^{c}}{\mathrm{e}^{d}}.$\\
    
When studying functions or properties defined on real numbers, it is usual to consider notions such as symmetry, linearity, etc. In the context of analogical proportions over real numbers, symmetry is satisfied by the two classical proportions.
Linearity, on the other hand, is satisfied by the arithmetic proportion in the following sense:
\begin{Prop}\label{linea}
$$\forall \alpha, \beta , a, b, c, d, a', b', c', d' \in \mathbb{R}:$$
$$(a:b::_{ari}c:d) \wedge (a':b'::_{ari}c':d') \implies 
(\alpha a + \beta a') : (\alpha b + \beta b')::_{ari}(\alpha c + \beta c'):(\alpha d + \beta d')$$
\end{Prop}
{\bf Proof:} Obvious due to the definition of the arithmetic proportion.\hfill$\Box$\\

This has immediate consequences such as:
\begin{itemize}
    \item $(a:b::_{ari}c:d) \wedge (a':b'::_{ari}c':d') \Longrightarrow (a\! + \! a') : (b \!+ \! b')::_{ari}(c \! +\!  c'):( d\! + \! d')$ $ (\alpha=\beta=1)$
    \item $(a:b::_{ari}c:d) \implies - a: - b::_{ari}- c: -d $ $(\alpha=0, \beta=-1)$
    \item $(a:b::_{ari}c:d) \implies \forall r \in \mathbb{R},   r- a: r - b::_{ari}r - c:r-d \mbox{ (because $r:r::_{ari}r:r$)}$
\end{itemize}
The third implication can be viewed as a counterpart of code independence in the Boolean case. This property for $r=1$ will also have a natural interpretation when probabilities are considered instead of basic numerical values.\\
Regarding geometric analogy, we have a partial form of  linearity such as:
$$\forall k \in \mathbb{R}^*, a:b::c:d \implies ka: kb::_{geo} kc :kd$$

Moreover, this is a particular case (since $k:k::_{geo} k:k$) of a more general statement: 

\begin{Prop} \label{Prod} $\forall  a, b, c, d, a', b', c', d \in \mathbb{R}$, if $ a  : b  ::_{geo} c  : d $ and $a'  : b'  ::_{geo} c' : d'$ then $a  a' : b b'  ::_{geo} c c' : d d'$.
 \end{Prop}
{\bf Proof:} Obvious due to the definition of the geometric proportion.\hfill$\Box$\\

Lastly, let us recall that the notation $a : b :: c : d$ was used (at least) until Gaspard Monge \cite{aicom/PradeR21} to mean $\frac{a}{b}=\frac{c}{d}$
(the notation `::' was introduced by the English mathematician William Oughtred at the beginning of the 17th century). A proof of the Pythagorean theorem in terms of geometric proportions is given as an example to illustrate their expressive power.

\begin{Example}
Consider a right triangle ABC at C, and the height from C, as shown in Figure \ref{Pytha}.
The angles $\widehat{HAC}$ and $\widehat{HCB}$ are equal because they have the same complement $\widehat{ACH}$. Triangles $ABC$, $ACH$ and $CBH$ are therefore similar.

\begin{figure}
    \caption{Pythagorean theorem}
    \label{Pytha}
    \centering
\begin{tikzpicture}[scale=1 ]
\centering
\coordinate [label=left:$A$] (A) at (-2cm,-1.cm);
\coordinate [label=right:$B$] (B) at (2.33cm,-1.0cm);
\coordinate [label=above:$C$] (C) at (1cm,1.0cm);
\coordinate [label=below:$H$] (H) at (1cm, -1cm);
\draw (A) -- node[sloped,above] {b} (C) -- node[sloped,above,] {a} (B) -- node[below] (c) {c} (A);
\draw[dashed] (C) -- node[sloped,above] {h} (C|-A) ;
 \draw [dashed] (B) --  node [sloped, above] (x) {x} (C|-A);

\end{tikzpicture}
 \end{figure}
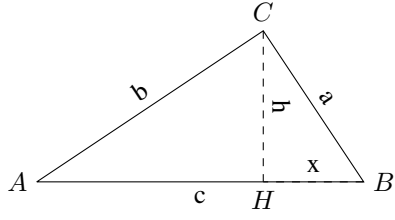
In Figure \ref{Pytha}, two sides of one of the three right triangles are therefore proportional to the two corresponding sides of another right triangle, and therefore form a geometric proportion, in terms of lengths. Thus one can write 

$CA : AB :: AH : CA$, i.e., $b : c :: c - x : b$
which leads to $b^2 = c^2 - cx$.

But we also have  $AB : CB :: CB : BH$, i.e.,  
 $c : a :: a : x$ and thus $cx = a^2$.

\noindent Hence we obtain $c^2 = a^2 + b^2$. 

Note the use of a continuous proportion  ($c : a :: a : x$).
\hfill $\Box$
\end{Example}

\subsection{Review of existing proposals} \label{existing}
We can distinguish between two types of proposals. On the one hand, there are multivalued extensions of the Boolean analogical proportion \cite{fss/DuboisPR16} that allow us to evaluate on $[0, 1]$ to what extent an analogical proportion between four numbers in $[0, 1]$ approximately holds. On the other hand two general frameworks \cite{LepCou2024}, \cite{sum/PradeR24} have recently been proposed to define whether or not an analogical proportion holds between four numbers in $[0, 1]$. 
These frameworks include arithmetic and geometric proportions as important special cases. 
\subsubsection{Multiple-valued logic extensions}.  
There exist mainly two types of multiple-valued logic extension of Boolean formula (\ref{formuledif})  to real numbers $a, b, c, d$ in the interval $[0, 1]$; see \cite{DubPraRicFSS2016} for justifications. There are respectively called liberal and conservative, denoted respectively with an `L' and a `C' subscript. They are given by:
\begin{equation}\label{form1extension}
a : b ::_{L} c : d =
\begin{cases}

1- \mid(a + d) - (b+c)\mid, \\ \quad \mbox{ if }a \geq b \mbox{ and } c \geq d, \mbox{ or } a \leq b \mbox{ and } c \leq d\\

1- \max(\!\mid \!a - b\!\mid,  \mid \!c - d\!\mid), \\ \quad \mbox{ if } a \leq b \mbox{ and } c \!\geq d, \mbox{ or } a \geq b \mbox{ and } c \leq d
\end{cases}
\end{equation} \vspace{-0.3cm}
\begin{equation}\label{form2extension}
a : b ::_C c : d = \min(1-|\max(a,d)-\max(b,c)|,1-|\min(a,d)-\min(b,c)|)
\end{equation}

When $a, b, c, d \in \mathbb{B}=\{0, 1\}$, we recover the truth table of (\ref{formuledif}) in both cases. $a : b ::_{L} c : d =1 $ if and only if $a-b=c-d$. This corresponds to the arithmetic  proportion. Thus, the expression (\ref{form1extension}) agrees with analogical dissimilarity since we have $AD(a, b, c, d) = 0 \Longleftrightarrow  a : b ::_L c : d =1$.

Regarding (\ref{form2extension}), $a : b ::_C c : d =1$ if and only if ($a=b$ and $c=d$) or ($a=c$ and $b=d$). This corresponds to the 3 possible patterns that make true a nominal analogical proportion
($(a, b, c, d) \in \{(s,s,s,s),(s,t,s,t),(s,s,t,t)\}$). Note that we only have $a : b ::_C c : d =1  \Longrightarrow AD(a, b, c, d) = 0$.  

One advantage of the   formulas (\ref{form1extension}) or (\ref{form2extension}) is that they offer a version that is graded (in $[0, 1]$) of the analogical proportion,  which enables us to have ``approximate'' proportions that hold to some degree.

\subsubsection{Generalized means-based definition}. Recently, an appealing  proposal \cite{LepCou2024,LepCou2025} has defined an analogical proportion between real numbers, not necessarily in $[0, 1]$, using generalized means. A generalized mean applied to two positive real numbers $x, y$ is a two-place function defined by  $M_p(x, y) = [\frac{1}{2} (x^p  + y^p)]^{1/p}$ where $p \in \mathbb{R}$ is the exponent of the mean. For  $p= -1, 1, 2$,  we 
obtain the harmonic, arithmetic, and quadratic mean respectively. The geometric mean is obtained as a limit for $p=0$.  We have $\forall p
\in \mathbb{R}, \min(x,y) \leq  M_p(x, y) \leq \max(x, y)$, and the bounds are reached for  $p=-\infty $ and $p=+\infty$  respectively. The definition of analogical proportion between four positive real numbers is then:
\begin{Definition}
Let $a, b, c, d  \in \mathbb{R}^+$,  $a, b, c, d $ are said to satisfy an analogical proportion if and only if $\exists p \in \mathbb{R}, M_p(a,d) = M_p(b,c)$.
\end{Definition}
 It can be easily seen that this definition ensures the satisfaction of the three postulates of analogical proportions. In \cite{LepCou2024}, it is established  that  given $a < b < c < d $,  $\exists !p  \mbox{  such that  }  M_p(a,d) = M_p(b,c) $\footnote{However, note that although an analogical proportion has 8 equivalent forms (see beginning of Section \ref{nominal}), it may be case that none of them are in increasing  order.}. Since $M_p(a,d) = M_p(b,c)$ amounts to write $a^p + d^p = b^p + c^p$, it means that $a^p, b^p, c^p, d^p$ form an arithmetic proportion. 
 Note that this definition, as the next one,  is not graded, but parameterized. 

\subsubsection{Triangular norm-based definition} Another definition \cite{sum/PradeR24}, which is also in agreement with  the three postulates of analogical proportions, is based on triangular norms and conorms\footnote{A  triangular norm  $T$ is a binary operation on $[0, 1]$, i.e., $T: [0, 1] \times [0, 1] \to [0, 1]$, that satisfies the following conditions, it is: i) 
commutative: $T(a, b) = T(b, a)$; ii)
associative: $ T(a, T(b, c)) = T(T(a, b), c)$; iii) non-decreasing in both arguments: $T(a, b)  \leq T (a', b')$ if $a  \leq a'$ and $b  \leq b'$; iv) 
such that $T (a, 1) = T (1, a) = a.$ 
 A binary operation $S$ on [0, 1] is called a triangular conorm if it satisfies the same properties as the ones of a triangular norm except for the boundary conditions, which become  $S (0, a) = S (a, 0) = a. $}. The definition is as such:
\begin{Definition} \label{tnorm}
Let $a, b, c, d $ be four numerical values in [0, 1].  $a, b, c, d $ are said to satisfy an analogical proportion based on the triangular norm $T$ (and its dual triangular conorm $S_T$), iff $ T(a, d) =T(b,  c)$ and $S_T(a, d) = S_T( b, c))$ \cite{sum/PradeR24}.
\end{Definition}
where $S_T(x, y) = 1 - T(1 - x, 1 - y)$. For $T(x, y) = xy$ in Definition \ref{tnorm}, $a, b, c,  d$ form both an arithmetic proportion, and a geometric proportion, which forces $(b,c)=(a,d)$ or $(b,c)=(d,a)$;  this corresponds to the nominal patterns of subsection \ref{nominal}; see also Proposition \ref{arigeo} below. For the {\L}ukasiewicz  triangular norm $T(x, y) = \max(0, x+y -1)$ in Definition \ref{tnorm}, $a, b, c, d $ form an arithmetic proportion. The use of a parameterized family of triangular norms enables us to have a parameterized definition of an analogical proportion between numbers in $[0, 1]$; see \cite{sum/PradeR24}.

\subsection{General definition}  \label{general}
An analogical proportion between real numbers $a,b,c,d$ is  defined in the last two  proposals above by equating an expression involving only the extreme elements $(a,d)$ with the same expression involving only the mean elements $(b,c)$ (e.g., $a+d=b+c$ for arithmetic proportion).
This observation raises at least two questions that we address now: What are the minimal requirements on a 
two-place  function $F$ for defining
a numerical analogical proportion denoted $a:b::^Fc:d$ as $F(a,d)= F(b,c)$? Can we consider combinations of numerical proportions other than the arithmetico-geometric one? \vspace{-0.4cm}
\subsubsection{General numerical definition} For the first question, we get the following:
\begin{Proposition} Let $a, b, c, d \in \mathbb{R} $, and $F$  a two-place function. Then 
$a:b::^Fc:d$ $\mbox{ iff } F(a,d)= F(b,c)$ 
defines an analogical proportion over $\mathbb{R}$ iff $F$ is symmetric, i.e., {$\forall x, y \in \mathbb{R}, F(x,y)=F(y,x)$}.
\end{Proposition}
{\bf Proof:} 
1. Having an analogical proportion implies the symmetry of $F$: the reflexivity of the proportion means 
$a:b::^Fa:b$ is true, and then $F(a,b)= F(b,a)$. 

2. The symmetry of $F$ implies we have an analogical proportion:  

      - reflexivity: $a:b::^Fa:b$ is true if $F(a,b)= F(b,a)$, which is true due to symmetry of $F$.
   
    - symmetry: $a:b::^Fc:d \implies c:d::^Fa:b$ is true due to symmetry of $F$ and $=$: $F(c,b)=F(b,c)=F(a,d)=F(d,a)$.
    
    - central permutation: $a:b::^Fc:d \implies a:c::^Fb:d$ is also true because of symmetry of $F$: then $F(a,d)= F(b,c)=F(c,b)$ \hfill $\Box$\\

This provides a very general process to build analogical proportions over $\mathbb{R}$ or to check that a given proportion is an analogy. 
When considering $\mathbb{B}$ as a subset of
$\mathbb{R}$, $F$ being symmetric ensures $a:b::^Fc:d$ holds for the 6 valid valuations of Table \ref{truthTableAnalogy}.
However it does not ensure that the 10 remaining invalid valuations 
$0001,0010,0100,1000,1110,1101,1011,0111,1001,0110$ do not hold.
Some requirements are necessary to fully fit with the Boolean case:
\begin{Proposition} $a:b::^Fc:d$ fits with the Boolean case iff 
$F(1,1) \neq F(0,0)$, $F(1,0) \neq F(1,1)$, and $F(1,0) \neq F(0,0)$. 
\end{Proposition}
{\bf Proof:} i) $F(1,1) \neq F(0,0)$ forbids $1:0::^F0:1$ and $0:1::^F1:0$ to hold.

 ii) $F(1,0) \neq F(1,1)$ forbids $1:1::^F1:0$, $0:1::^F1:1$, $1:1::^F0:1$, $1:0::^F1:1$ to hold.

    iii) $F(1,0) \neq F(0,0)$ forbids $1:0::^F0:0$, $0:0::^F0:1$, $0:0::^F1:0$, $0:1::^F0:0$ to hold.
\hfill $\Box$\\
Clearly, $F(x, y) =x + y$ is the simplest function satisfying the above requirements; it leads to the arithmetic proportion.

\subsection{Arithmetico-geometric analogical proportions}  \label{ari+geo}
We end this section by a proposition showing that numerical analogical proportions that are \textit{both} arithmetic and geometric are highly constrained. However we shall see in Section \ref{probadis} that this result still  opens the possibility of an interesting definition of analogical proportions between probability densities.

\begin{Proposition} \label{arigeo} Let $a, b, c, d$ be four real numbers. Then these numbers are in \emph{arithmetico-geometric} analogical proportion, denoted
$a : b ::_{arigeo} c : d $ if  we have
$a  - b  = c  - d $ and
$\frac{a }{b }= \frac{c}
{d}$. The only 4-tuples of real numbers $a, b, c, d$ such that $a : b ::_{arigeo} c : d$ holds, are such that:
 
 - { case 1:} either   $a=b$ and $c=d$, 
 
 - { case 2:} or   $a=c$ and $b=d$.   
 
\noindent and these are the only solutions.
\end{Proposition} \vspace{-0.2cm}
{\bf Proof} 
Setting \(a  + d = s \) and \(a d = p  \), we obtain \(d  = s  - a  \), then \(a (s- a) = p  \), that is, $a^2 -s\cdot a  + p  =0$.
Thus, $a$ must be a solution to the quadratic equation \cite{sum/PradeR24}:  
$x^2 - s x + p = 0$. 
Its discriminant is equal to $\Delta=s^2 -4p= (b + c)^2 - 4b\cdot c= (b - c)^2$. The two solutions are therefore $a =\frac{s \pm \sqrt{\Delta}}{2}=\frac{b  +  c \pm (b  - c)}{2}$. This gives $a=b$ or $a=c$.
Thus, we should have \((a, d)= (b, c) \) or \((d, a) =(b, c) \).  
$\hfill\Box$\\
The arithmetico-geometric analogical proportions are thus valid for the same kind of patterns as Boolean or nominal analogical proportions.

\begin{Example}

Let us consider the following  vectors $a, b, c, d$ with 4 components ($n=4$):

$a_1= 0.15; a_2=0.10; a_3=0.70; a_4 = 0.05.$

$b_1=0.55 ; b_2= 0.20; b_3= 0.20; b_4 = 0.05.$

$c_1= 0.25; c_2=0.10 ; c_3=0.60; c_4 = 0.05.$

$d_1=0.65 ; d_2=0.20; d_3= 0.10; d_4 = 0.05.$\\
{ Componentwise,}
we have $a_i- b_i = c_i - d_i$  for $i= 1, 2, 3, 4$. We also have an arithmetico-geometric proportion for $i= 2$, and $i=4$. This is acknowledged by the facts that $a_1 : b_1 ::_C c_1 : d_1 = 0.9 = a_3 : b_3 ::_C c_3 : d_3 $, while we have $a_2 : b_2 ::_C c_2 : d_2 =a_4 : b_4 ::_C c_4 : d_4 =1$. Changing, $c_1$, $d_1$, $c_3$, $d_3$ into 0.15, 0.55, 0.70, 0.20 respectively would lead to an arithmetico-geometric proportion between the four distributions. \hfill $\Box$
\end{Example}

\vspace{-0.3cm}
\subsubsection{Other numerical proportions}
It can be shown that a  result similar to Proposition \ref{arigeo} is obtained if we replace the geometric  proportion by other proportions.

\begin{Proposition} \label{new}  The only 4-tuples of real numbers $a, b, c, d$ such that $a  - b  = c  - d $ and $a : b ::^{F} c : d$ hold, when $F(x,y)=\frac{2xy}{x+ y}$ (harmonic proportion), or $F(x,y)=\sqrt{\frac{1}{2}(x^2 + y^2)}$ (quadratic proportion),
or $F(x,y)= \sqrt[3]{\frac{1}{2}(x^3 + y^3)}$ (cubic proportion) 
are such that $(a,b,c,d) \in \{(s,s,s,s),(s,s,t,t),(s,t,s,t)\}$.
\end{Proposition}

{\bf Proof:} 
 The proof of Proposition \ref{arigeo}
 relies on the fact that $a$ must be a solution to the quadratic equation   
$x^2 - s x + p = 0$ where $s= a + d = b + c$ and $p= ad =bc$.

If we take the harmonic proportion, we   see that we should have $a+d=b+c$ and $ ad =bc$ (since $\frac{2ad}{{a}+ {d}}= \frac{2bc}{{b}+ {c}}$ entails $\frac{ad}{{a}+ {d}}= \frac{2bc}{{a}+ {d}}$ and $\frac{ad}{{a}+ d}= \frac{ad}{{b}+ {c}}$). We are back to Proposition \ref{arigeo}. 

For the quadratic proportion, we have $a^2 + d^2= b^2 + c^2$, which can be written $(a + d)^2 - 2 ad= (b + c)^2 - 2bc$, and we are back to Proposition \ref{arigeo} since $a+d=b+c$, and thus we get $ad =bc$. 

For the cubic proportion, observe that $a^3 + d^3 = (a + d)^3 -3 ad(a +d)$, and we can come back to Proposition \ref{arigeo}. \hfill $\Box$\\

As a conclusion, two options emerge for numerical analogical proportions: the arithmetic proportion, and the so-called arithmetico-geometric proportion. These are the two options we shall consider for defining analogical proportions between probability
distributions in subsection \ref{probadis}.

\section{Analogical proportions and probability}\label{proba}
This section studies analogical proportions between probabilities. We first consider the case of simple probability values, before dealing with probability distributions. We will discuss the arithmetic and geometric proportions between discrete probability distributions, before considering the arithmetico-geometric proportion. Then we will see how the analogical dissimilarity can apply to probability distributions. Finally, extensions to continuous probability densities are briefly discussed.
\subsection{Analogical proportion between probability values}\label{probaval}
In this subsection,   we are interested in analogical proportions between real numbers which are probabilities, which therefore have values between 0 and 1. 
{For instance,} consider the case of four probabilities that refer to four different populations and that concern a certain value of the same attribute
for the elements of these four sets. { In that context,} $a$, $b$, $c$, $d$ are the probabilities that an element of a set { $S_a$, respectively $S_b$, $S_c$, $S_d$}, takes a value  $v_i$ for an attribute $i$. We can of course be interested in different attributes simultaneously. This is the case, for example, if we have a collection of examples and can perform statistics on the attribute values for four different data groups.
In the following, we apply   arithmetic and geometric analogical proportions (denoted $::_{ari}$ and $::_{geo}$ respectively) to probability values.

The following property, which is straightforward to verify, ensures that if an analogical proportion holds between four probabilities, it also holds when considering their corresponding complementary events:
\begin{Prop}\label{Prop1}  
If $a : b ::_{ari} c:d $ holds then  
 $1-a : 1-b ::_{ari} 1-c: 1-d $ holds. 
\end{Prop}

This desirable property is not   true for $::_{geo}$.
However, it does hold if the probabilities are also in arithmetic proportion. 

\begin{Prop} 
 If $a : b ::_{ari} c:d $ and   $a : b ::_{geo} c:d $ hold then   
 $1-a : 1-b ::_{geo} 1-c: 1-d $ holds. 
\end{Prop}
{\bf Proof } 
The equality $(1\!-\!a)(1\!-\!d)\! =\! (1\!-\!b)(1\!-\!c)$ holds if $ad \! = \! bc$ and $a+d \! =\! b+c$. 

\hfill$\Box$

\vspace{0.2cm}
\noindent This  can be considered
as the counterpart of the code independence property, valid in Boolean logic:  
 $a : b :: c : d \   \Rightarrow \neg a : \neg b :: \neg c: \neg d $,
where $a,b,c,d$ represent  Boolean values.

\vspace{0.1cm}\textit{As in the Boolean and nominal cases, there is not always  an $x$ such that $a : b ::_{ari} c:x$
is true.} Indeed
$$0\leq a\leq 1, 0\leq b\leq 1, 0\leq c\leq 1 \not\Rightarrow 0\leq x=b + c - a\leq 1. $$
Similarly, there is not always an  $x$ such that $a : b ::_{geo} c:x$ holds, since
$$0\leq a\leq 1, 0\leq b\leq 1, 0\leq c\leq 1 \not\Rightarrow 0\leq x=\frac{bc}{a} \leq 1. $$
In both cases, if $b$ and
$c$ are large then we must have $a$ large, and if $a$ is small, we must have $b$ or $c$ small for there to be a solution.

\subsection{Arithmetic vs. geometric  proportions between probability distributions} \label{probadis}
Instead of vectors as in Definition \ref{DefVec},
we consider the case where $a, b, c, d$ denote probability distributions over the same finite domain. 
The only difference with the vectors case, is that we impose {every $a_i,b_i,c_i,d_i$ to be positive and} $\Sigma_{i=1}^n a_i=1 $, $\Sigma_{i=1}^n b_i=1 $, $\Sigma_{i=1}^n c_i=1 $, $\Sigma_{i=1}^n d_i=1 $.
In that context, { $a, b, c, d$ should be considered as associated to 4 discrete random variables $X_a, X_b, X_c, X_d$  {over the same sample space} and taking a finite set of values $S=\{v_i | i \in [1,n]\}$ where $a_i$ is just $P(X_a=v_i)$.}

We now examine the arithmetic and the geometric  definitions of analogical proportions between probability distributions and study their respective properties.
\subsubsection{Arithmetic definition }
We will first examine the definition based on the arithmetic proportion.
\begin{Definition}\label{ProbAri}  Let $a, b, c, d$ be four probability distributions on a finite domain $X= \{x_1, \ldots, x_n\}$. Then these distributions are in \emph{arithmetic} analogical proportion, denoted
$a : b ::_{ari} c : d $ 
if for all $i$ we have
$a_i - b_i = c_i - d_i$,  

\noindent where $a_i=a(x_i)$, $b_i=b(x_i)$, $c_i=c(x_i)$, $d_i=d(x_i)$.
\end{Definition}
It is clear that $::_{ari}$ satisfies the three postulates of analogical proportions. There exist probability distributions that satisfy this definition. Indeed we have: 
\begin{Proposition} \label{ProposAri} Let $a, b, c$ be three 
 probability distributions,  if $0 \leq c_i + b_i - a_i \leq1$, $\forall i$,
then there exists a unique    probability distribution $d$ such that $a : b ::_{ari} c : d$.
\end{Proposition}
\noindent {\bf Proof } The condition $ 0 \leq c_i + b_i - a_i \leq1$ ensures that $0 \leq d_i \leq1$. Since $\forall i, a_i - b_i= c_i - d_i$, member-by-member addition of these equalities shows that $\Sigma_{i=1, n} d_i=1$ since
$\Sigma_{i=1, n} a_i=1; \Sigma_{i=1, n} b_i=1$ and $\Sigma_{i=1, n} c_i=1$. \hfill $\Box$

The condition $\forall i, 0 \leq c_i + b_i - a_i \leq1$ is necessary for $d$ to be a probability distribution, as the following counterexample shows.
\begin{CE}
Let us take $n=2$. $a_1= 0.7 , a_2 = 0.3$; $b_1= 0.3 , b_2 = 0.7$; $c_1= 0.2 , c_2 = 0.8$. 

We obtain $d_1= - 0.2, d_2 = 1.2$. \hfill $\Box$
\end{CE}

The Definition \ref{ProbAri} covers the deterministic case:
\begin{Ob} 
In the deterministic case, the probability distributions are such that $\exists i, e_i = 1$ and $\forall j \not = i, e_j =0$, for $e \in \{a, b, c, d\}$.
There are then three possibilities for $a : b ::_{ari} c : d$ to hold:
\begin{itemize}
    \item $\exists i, a_i= b_i= c_i = d_i =1$;
    \item $\exists i, a_i= b_i=1$ and $\exists j \not = i,   c_j = d_j =1$;
    \item $\exists i, a_i= c_i=1$ and $\exists j\not = i,   b_j = d_j =1$.
    \hfill $\Box$
\end{itemize}
\end{Ob}
Let us now give some examples, not extreme like the previous ones, of probability distributions for which we will obtain $a : b ::_{ari} c : d $.
\begin{Example} \
\begin{enumerate}
\item Using Property \ref{Prop1}, we see that as soon as $\exists i, \exists j\not = i, \mbox{ s. t. } 
a_i : b_i ::_{ari} c_i : d_i $,  with 
$a_j = 1-a_i, b_j = 1-b_i, c_j = 1-c_i, d_j = 1-d_i $ and $a_k = b_k= c_k= d_k =0, \forall k \neq i,j$,
we have $a : b ::_{ari} c : d $.
 
\item In fact if $a_i : b_i ::_{ari} c_i : d_i$, for any real number $\lambda$ we have   $\lambda - a_i : \lambda - b_i ::_{ari} \lambda - c_i : \lambda - d_i$. Thus by taking positive values whose sum $\lambda$ does not exceed 1, we can construct pairs$((a_i, b_i, c_i, d_i), (a_j, b_j, c_j, d_j))$ forming two arithmetic proportions s. t. $a_i + a_j  
= b_i + b_j=c_i + c_j= d_i + d_j =\lambda $, as in the following example where we have partial probability allocations $\lambda= 0.4$ and $\lambda'=0.5$ completed by a quadruplet $(a_k, b_k, c_k, d_k)$ of values all equal to  $ 1- \lambda - \lambda' = 1-(0.4 + 0.5) = 0.1$ (here $\lambda$ pertains to $i=1,2$, $\lambda'$ to $i=4,5$, and $k=3$):

$\quad\quad a_1= 0.1; a_2=0.3 ; a_3=0.1 ; a_4=0.2 ; a_5=0.3.$

$\quad\quad b_1=0.3 ; b_2= 0.1; b_3= 0.1; b_4=0.3 ; b_5= 0.2.$

$\quad\quad c_1= 0.2; c_2=0.2 ; c_3=0.1 ; c_4=0.4 ; c_5=0.1.$

$\quad\quad d_1=0.4 ; d_2=0 ; \ \ d_3= 0.1; d_4=0.5 ; d_5= 0.$

\vspace{0.1cm}\noindent We check that $a : b ::_{ari} c : d $ holds between probability distributions that are all different.
\item 
For $a : b ::_{ari} c : d$ to be satisfied, it is not necessary to proceed as before, that is, to construct pairs of quadruplets whose term-by-term sum is constant.
This is shown by the following two examples
for $n=3$ and $n=4$: 

$\quad\quad a_1= 0.3; a_2=0.2; a_3=0.5.$

$\quad\quad b_1=0.5 ; b_2= 0.1; b_3= 0.4.$

$\quad\quad c_1= 0.4; c_2=0.2 ; c_3=0.4.$

$\quad\quad d_1=0.6 ; d_2=0.1; d_3= 0.3.$

\item

$\quad\quad a_1= 0.1; a_2=0.2 ; a_3=0.4 ; a_4=0.3.$

$\quad\quad b_1=0.3 ; b_2= 0.3; b_3= 0.2; b_4=0.2.$

$\quad\quad c_1= 0.2; c_2=0.2; c_3=0.2; c_4=0.4.$

$\quad\quad d_1=0.4 ; d_2=0.3 ; d_3= 0; d_4=0.3.$ \hfill $\Box$
\end{enumerate}
\end{Example}

These examples show that there exist non-trivial distributions satisfying Definition \ref{ProbAri}. Of course, given three distributions $a, b, c$, there does not always exist a distribution $d$ such that $a : b ::_{ari} c : d$ holds, since we must have the  inequalities 

$\quad\quad\quad\quad\quad\quad\quad\quad\quad\quad\forall i, 0 \leq c_i + b_i - a_i \leq1 \quad\quad(4)$. 

\begin{Ob} Note, however, that given any two distributions $a$ and $b$, it is always possible to find a distribution $c$ such that the  inequalities (4) are satisfied for each $i$. We can then find a distribution $d$ in arithmetic analogical proportion with $a, b, c$. \hfill $\Box$
\end{Ob} 

The linearity offered by Definition \ref{ProbAri} guarantees the transfer of analogical proportions 
between probability distributions to  analogical proportions 
between probabilities associated to an event, as observed below.
\begin{Ob} \label{event} Thanks to Property \ref{linea}, the following holds:
 If $a_i : b_i ::_{ari} c_i : d_i$ and $a_j : b_j ::_{ari} c_j : d_j$ then $a_i + a_j : b_i + b_j ::_{ari} c_i + c_j : d_i + d_j$.

Thus, if the four probability distributions $a, b, c, d$ are 
such that $a : b ::_{ari} c : d$ holds, 
then an arithmetic analogical proportion holds true between the probabilities of any event calculated from these distributions (which correspond to the sum of the probabilities in the distribution corresponding to the event). \hfill $\Box$
\end{Ob} 

In fact, the constraint $a : b ::_{ari} c : d$ on four distributions has a direct consequence in terms of distance between probabilities.
Since we are working with discrete finite distributions, a natural and well-known metric on this space is the {\it Total Variation Distance} defined in that case as:
$$TV(a,b) = \frac{1}{2}||a - b||_1$$
A straightforward consequence is that arithmetic analogical proportion between four discrete finite probability distributions preserves total variation distance:
\begin{Proposition}\label{TV}
If $a : b ::_{ari} c : d$ then TV(a,b) = TV(c,d).
\end{Proposition}
{\bf Proof}:
In fact $TV(a,b)=\frac{1}{2}||a - b||_1=\frac{1}{2}\Sigma_{i=1}^n |a_i-b_i|$. But by Definition \ref{ProbAri}, $\forall i \in [1,n],(a_i-b_i)=(c_i-d_i)$ then $\forall i \in [1,n],|a_i-b_i|=|c_i-d_i|$.
Then $TV(a,b)=\frac{1}{2}\Sigma_{i=1}^n |c_i-d_i| = TV(c,d)$.\hfill$\Box$\\

We now consider the definition of analogical proportion between probability distributions based on the geometric proportion.
\subsubsection{Geometric definition }
It is natural to wonder what would happen to a definition
similar to Definition \ref{ProbAri} but in terms of geometric proportion, that is:
\begin{Definition} Let $a, b, c, d$ be four probability distributions over a finite domain $X= \{x_1, \ldots, x_n\}$. Then these distributions are in \emph{geometric} analogical proportion, denoted $a : b ::_{geo} c : d $
if for all $i$ we have
$\frac{a_i}{b_i}  = \frac{c_i}{d_i}$,  

\noindent where $a_i=a(x_i)$, $b_i=b(x_i)$, $c_i=c(x_i)$, $d_i=d(x_i)$.
\end{Definition}
Unfortunately, as the two following counterexamples show that the counterpart of Proposition \ref{ProposAri} for geometric proportion is false. Having $\Sigma_{i=1}^n a_i=1 $,$\Sigma_{i=1}^n b_i=1 $, $\Sigma_{i=1}^n c_i=1 $ with  $\forall i,  \frac{b_ic_i}{a_i}   \leq 1$ 
is not sufficient to guarantee $\Sigma_{i=1}^n d_i=1 $ when $a : b ::_{geo} c : d$ holds.
\begin{CE}
Let us take $n=3$. 
We enforce the constraints  $a_1 + a_2 + a_3= 1$, $b_1 + b_2 + b_3= 1$,  $c_1 + c_2 + c_3= 1$, with the condition $b_i\cdot c_i \leq a_i$ in the two sets of values given below;
moreover we should have $a_1\cdot d_1= b_1\cdot c_1$,
$a_2\cdot d_2= b_2\cdot c_2$, $a_3\cdot d_3= b_3\cdot c_3$.
\begin{itemize}
    \item $b_1= 0.2; b_2= 0.3; b_3= 0.5$, $c_1= 0.4;c_2= 0.3 ; c_3= 0.3$ thus $b_1\cdot c_1=0.08 ; b_2\cdot c_2= 0.09;  b_3\cdot c_3=0.15$. We take $a_1= 0.3; a_2= 0.4; a_3= 0.3$. This gives $d_1 =\frac{4}{15}; d_2 =\frac{9}{40}; d_3 =\frac{1}{2}   $, and finally $\Sigma_i d_i =0.991666 < 1$
    \item $b_1= 0.1 ;b_2= 0.5; b_3=0.4 $, $c_1=0.2 ;c_2=0.2 ; c_3= 0.6$. Thus $b_1\cdot c_1= 0.02 ; b_2\cdot c_2=0.10 ;  b_3\cdot c_3=0.24$. Keeping the same $a$ as above ($a_1= 0.3; a_2= 0.4; a_3= 0.3$), we get $d_1 =\frac{1}{15}; d_2 =\frac{1}{4} ; d_3 =\frac{4}{5}$ and then $\Sigma_i d_i =1.116666  > 1$.
    
    \hfill $\Box$
    \end{itemize}
\end{CE}

\begin{Rem}
In the same spirit as for probabilities, one could ask the question of defining an analogical proportion between \emph{possibility} distributions \cite {ZAd78,DP1988} $a$, $b$, $c$, $d$, where
the normalization is expressed by
$\max_i a_i= 1, \max_i b_i= 1, \max_i c_i= 1, \max_i d_i= 1$, the $a_i, b_i, c_i, d_i$ being between 0 and 1. Since a possibility distribution $e$ can be seen in terms of its $\alpha$-level cuts $e_\alpha=\{i \ | \ e_i\geq \alpha\} $ which are nested subsets $e_\alpha \subseteq e_\beta$ if $\alpha \geq\beta$, we can reduce this to the Boolean approach for each $\alpha$-level cut; see subsection \ref{dia} and  \cite {diagrams/PradeR24}. 
\hfill $\Box$
\end{Rem}

\subsection{Analogical dissimilarity between probability distributions} \label{ad-extension}
 As discussed in subsection \ref{Boo}, analogical dissimilarity $AD$ has been initially introduced to quantify how far 4 Boolean $a,b,c,d$ are from building a valid analogical proportion. The definition of $AD$ extends naturally to numerical values and more generally, to real valued vectors in $\mathbb{R}^n$ with:
$$AD_p(a,b,c,d) = ||(a-b)-(c-d)||_p$$
where $||.||_p$ is any norm on $\mathbb{R}^n$ \cite{MicBayDelJAIR2008}. In the case of Euclidean norm, this value reflects in some way how far the vectors are from building a parallelogram.  
Obviously, when $n=1$ and using the norm $L_1$ norm, this is just the sum from $i=1$ to $i=n$ of  $|(a_i-b_i)-(c_i-d_i)|$ measuring the default of the $a_i,b_i,c_i,d_i$ to build an arithmetic proportion. When the 4 numbers $a_i,b_i,c_i,d_i$ belong to $[0,1]$, the analogical dissimilarity $AD(a_i,b_i,c_i,d_i)$ varies between $0$ and $2$ as it is the case for Boolean values (see section \ref{Boo}). It makes sense to divide this number by $2$ in order to stay between $0$ and $1$. 
Also, when considering vectors in $\mathbb{R}^n$, we leave the initial definition by getting the average value
of analogical dissimilarity componentwise, i.e., 
$$\frac{1}{2n}\Sigma_{i=1}^n |(a_i-b_i)-(c_i-d_i)| $$
This is just a linear rescaling of the $||.||_1$ norm, ensuring we compare numbers between $0$ and $1$:
$$AD(a,b,c,d) = \frac{1}{2n}||(a-b)-(c-d)||_1$$
Because $AD(a,b,c,d) = 0$ is equivalent to $a:b::_{ari}c:d$ and thanks to Proposition \ref{TV},
$AD(a,b,c,d) = 0$ implies $TV(a,b) = TV(c,d)$.
\subsection{Arithmetico-geometric proportion between probability distributions}\label{arigeodist}
In this subsection, we study a definition of the analogical proportion between probability distributions, more demanding than Definition \ref{ProbAri} because it combines arithmetic and geometric proportions.
\begin{Definition} \label{sophi}
Let $a, b, c, d$ be four probability distributions over a finite domain $X= \{x_1, \ldots, x_n\}$. Then these distributions are in \emph{arithmetico-geometric} analogical proportion, denoted
$a : b ::_{arigeo} c : d $ if for all $i$ we have
$a_i - b_i = c_i - d_i$ and
$\frac{a_i}{b_i}= \frac{c_i}
{d_i}$,
where $a_i=a(x_i)$, $b_i=b(x_i)$, $c_i=c(x_i)$, $d_i=d(x_i)$.
\end{Definition}
It is clear that $::_{arigeo}$ satisfies the three postulates of analogical proportions. Due to Proposition  \ref{arigeo} (see subsection \ref{ari+geo}), there exist  vectors $a, b, c, d$, which can all be different, such that  $a : b ::_{arigeo} c : d$ holds.  Despite the restrictive nature of Definition \ref{sophi}, there exist non-trivial probability distributions that satisfy it. 

{ \begin{Proposition} \label{ProposNomi} If $a, b, c$ represent probability distribution (i.e. $\Sigma_{i=1}^n a_i=1 $, etc.) on the same finite set $X$, then this is also the case for $d$, such that $a : b ::_{arigeo} c : d $ holds.
\end{Proposition}} 
{\bf Proof} 
Due to Proposition  \ref{arigeo}, 
$a, b, c, d$ are such that for each component $i$, 
 
 -  either   $a_i=b_i$ and $c_i=d_i$, 
 
 - or   $a_i=c_i$ and $b_i=d_i$.   
 
Thus, we should have \((a_i, d_i)= (b_i, c_i) \) or \((d_i, a_i) =(b_i, c_i) \).
So $a:b::_{arigeo}c:d$ iff for each component $i$, we have
($a_i=b_i$ and $c_i=d_i$), or
 ($a_i=c_i$ and $b_i=d_i$). 

 Note that if $n \geq 2 $, and $\exists j, a_j=b_j$, $c_j=d_j$, and  $\exists k, a_k=c_k$, $b_k=d_k$ then $a, b, c, d$ are all different.

Let $J = \{j  \ |\ a_j=b_j \} $ and $K= \{k  \ |\ a_k=c_k\} $. 

Then $\Sigma_{i=1}^n d_i= \Sigma_{j\in J} d_j + \Sigma_{k \in K} d_k $

\quad \quad \quad\quad\quad \ \ $=\Sigma_{j\in J} c_j + \Sigma_{k \in K} b_k $  (since $c_j =d_j $ if $j\in J$ and $b_k =d_k $ if $k\in K$)

\quad \quad \quad\quad\quad \ \ $= 1 -\Sigma_{k \in K} c_k +  1 - \Sigma_{j\in J} b_j$ (since $\Sigma_{i=1}^n b_i= \Sigma_{i=1}^n c_i=1$)

\quad \quad \quad\quad\quad \ \ $=1 -\Sigma_{k \in K} a_k +  1 - \Sigma_{j\in J} a_j$  ($a_k =c_k $ if $k\in K$, $a_j =b_j $ if $j\in J$)

\quad \quad \quad\quad\quad \ \ $=1$ (since $\Sigma_{i=1}^n a_i=1$)
\hfill$\Box$\\

Here is an example of an analogical proportion between probability distributions $a, b, c, d$, all different, obeying the Definition \ref{sophi}:
\begin{Example}

\ $a_1 = 0.1, a_2 = 0.3, a_3 = 0.2, a_4 = 0.3, a_5 = 0.1$

\quad\quad\quad\quad  $b_1 = 0.1, b_2 = 0.4, b_3 = 0.2, b_4 = 0.2, b_5 = 0.1$

\quad\quad\quad\quad $c_1 = 0.1, c_2 = 0.3, c_3 = 0.3, c_4 = 0.3, c_5 = 0$

\quad\quad\quad\quad $d_1 = 0.1, d_2 = 0.4, d_3 = 0.3, d_4 = 0.2, d_5 = 0$

 We can verify that $\Sigma_{i=1}^n a_i=1$, $\Sigma_{i=1}^n b_i=1$, $\Sigma_{i=1}^n c_i=1$, $\Sigma_{i=1}^n d_i=1$, and that for each $i$, we have both $a_i - b_i = c_i - d_i$ and $\frac{a_i}{b_i}= \frac{c_i}{d_i}$ (recall that we have taken the convention $\frac{0}{0}=1$).
 \hfill$\Box$
 \end{Example}
 
 Observe that for each component  $i$ we have three kinds of possible patterns for the probability values: $(s,s,s,s),(s,t,s,t)$, and $(s,s,t,t)$, these are precisely those already encountered for the nominal values in subsection \ref{nomi}.2. Note that the last two patterns must be present if we want to have distinct distributions, otherwise we have $a=b$ (and therefore $c=d$), or $a=c$ (and therefore $b=d$). The pattern $(s,s,s,s)$ may be absent. A similar situation had already been observed in the Boolean case \cite{ijar/BounhasP24}.

 It is easy to find probability distributions $a, b, c, d$ such that $a:b::_{arigeo}c:d$ holds. Indeed, we have the following result.
\begin{Proposition} \label{exist}
Given two different probability distributions $a$ and $b$, there always exists two
probability distributions $c$ and $d$ such that
$a:b::_{arigeo}c:d$ holds. If there exists at least one $i$ such that $a_i = b_i$, the four distributions can be different.
\end{Proposition} \vspace{-0.2cm}
{\bf Proof }   
Indeed, $a:b::_{arigeo}c:d$ entails $a_i=b_i$ and $c_i=d_i$, 
  or   $a_i=c_i$ and $b_i=d_i$.
If $a_i \neq b_i$, then $c_i = a_i$ and $d_i = b_i$. Note, since $\Sigma_{i \ |\ a_i = b_i} a_i= \Sigma_{i \ |\ a_i = b_i} b_i\triangleq \rho $, we have $\Sigma_{i \ |\ a_i \neq b_i} a_i= \Sigma_{i \ |\ a_i \neq b_i} b_i= 1- \rho $. So we also have $\Sigma_{i \ |\ c_i \neq d_i} c_i= \Sigma_{i \ |\ c_i \neq d_i} d_i= 1- \rho $.
For $i$ such that $a_i = b_i$, we share the remaining mass $\rho$ between these $i$'s so that 
 $c_i$ = $d_i$, and  $\Sigma_{i \ |\ a_i = b_i} c_i= \Sigma_{i \ |\ a_i = b_i} d_i= \rho$, 
  and then we have $\Sigma_{i=1, n}c_i=\Sigma_{i=1, n}d_i = 1 $.\hfill $\Box$\\

These probabilities $a, b, c, d$ satisfying $a:b::_{arigeo}c:d$ have, as we have just seen, a particular form. They satisfy a remarkable property in terms of Kullback-Leibler divergence, which, as we recall, evaluates the change between two distributions $a$ and $b$ by the expression $KL(a||b) = \Sigma_{i=1, n} a_i\log\frac{a_i}{b_i}$ (in the finite case).
In general, $KL(a||b)\neq KL(b||a)$. Then, we can state the following result:

{\begin{Proposition}\label{KL} Let $a, b, c, d$ be four probability distributions forming an arithmetico-geometric analogical proportion $a:b::_{arigeo}c:d$. Then we have:
$$KL(a||b) = KL(c||d)$$ where $KL$ is the Kullback-Leibler divergence.
\end{Proposition}  
{\bf Proof} 
Since $KL(a||b) = \Sigma_{i=1, n} a_i\log\frac{a_i}{b_i}$ and   $KL(c||d) = \Sigma_{i=1, n} c_i\log\frac{c_i}{d_i}$,
let us look at the different  terms $i$ : $a_i\log\frac{a_i}{b_i}$ and $c_i\log\frac{c_i}{d_i}$; 
due to Proposition \ref{arigeo}, we know that there are two cases:

- case (1): $a_i=b_i$ (and  $c_i=d_i$): both terms are zero; 
 case (2): $a_i=c_i$: then we have 
$a_i\log\frac{a_i}{b_i} = c_i\log\frac{c_i}{d_i}$  since $\frac{a_i}{b_i}= \frac{c_i}{d_i}$.
Hence $KL(a||b)=KL(c||d)$.}\hfill $\Box$\\

Since $::_{arigeo}$ satisfy the analogical proportion postulates, we also have 

\begin{Prop}   If $a:b::_{arigeo}c:d$ then i) $KL(a||c)=KL(b||d)$;\quad
ii) $KL(b||a)=KL(d||c)$; \quad  
 iii) $KL(d||b)=KL(c||a)$.
 \end{Prop}
\noindent {\bf Proof} i) thanks to Proposition \ref{KL} and central permutation; ii) since  $a:b::_{arigeo}c:d\Rightarrow b:a::_{arigeo}d:c$, by successive applications of the central permutation, symmetry, and again central permutation;  iii) the last equality, which corresponds to the permutation of the extremes, can be obtained by applying the second equality to the first one. \hfill $\Box$ \\

Finally, we have the following property for the arithmetico-geometric analogical proportion between probability distributions.
 
\begin{Prop}\label{arigeoind} 

If on the one hand four probability distributions $a, b, c,d$ over $S$ are such that $a : b ::_{arigeo} c : d$
and on the other hand four probability distributions $a', b', c', d'$  over $S'$
are such that $a' : b' ::_{arigeo} c' : d'$
then due to  Property   \ref{linea} for arithmetic proportions,   the probabilities of  two events $E_1, E_2$ calculated from these respective distributions, say $p_a, p_b, p_c, p_d$ and $p_{a'}, p_{b'}, p_{c'}, p_{d'}$ are such that 
$p_a : p_b ::_{arigeo} p_c: p_d$ and $p_{a'} : p_{b'}::_{arigeo} p_{c'} : p_{d'}$ (see Observation \ref{event}). 
But, due to Property \ref{Prod} for geometric proportions, we then have 
 $$p_a p_{a'} : p_b p_{b'} ::_{arigeo} p_c p_{c'} : p_d p_{d'}$$

This means, assuming independence of $E_1$ and $E_2$, that the probability distributions of their conjunction 
are also in arithmetico-geometric analogical proportion. 
  \hfill $\Box$
\end{Prop}

\subsection{Towards further developments}\label{further} 
In the previous section, the probabilities where associated with a finite set of values $X$ without any particular structure. We may wonder if the previous results can be extended to other kinds of probabilities, for instance, {defined via} continuous densities. In such a case, definitions extend straightforwardly: $a : b ::_{ari} c : d $ 
if for all $x \in X = \mathbb{R} $ we have
$a(x) - b(x) = c(x) - d(x)$ where the four densities $a, b, c, d$ obey  $\int_{\mathbb{R}} 
a(x)dx= \int_{\mathbb{R}} b(x)dx = \int_{\mathbb{R}} 
c(x)dx = \int_{\mathbb{R}} d(x)dx=1$. The geometrical proportion can be similarly extended in a pointwise manner.
{Some examples of an arithmetico-geometric analogical proportion between continuous distributions, defined using piecewise linear functions, are given in Annex.}

Another option for working with discrete or continuous distributions characterized by two parameters, such as binomial distributions
B(n, p) or Gaussian distributions $\mathcal{N}(\mu,\sigma)$, is to define analogical proportions directly between their parameters (possibly with two different numerical  proportions), rather than in a pointwise manner. 
For instance, consider $a, b, c, d$ be four Gaussian probability density functions, with respective means $\mu_a,\mu_b,\mu_c,\mu_d$ and standard deviations
$\sigma_a, \sigma_b, \sigma_c, \sigma_d$, and let us say that they form an arithmetico-geometric analogical proportion iff
$\mu_a-\mu_b=\mu_c-\mu_d$ and their variances are such that $\sigma^2_a /\sigma^2_b = \sigma^2_c / \sigma^2_d$. Then we have \cite{Soch}  $$KL(a||b)=\frac{1}{2}[\frac{\sigma_a^2}{\sigma_b^2} - ln(\frac{\sigma_a^2}{\sigma_b^2}) + \frac{(\mu_a - \mu_b)^2}{\sigma_b^2} -1],$$  This makes clear that $KL(a||b)=KL(c||d)$ only if $\sigma_b^2=\sigma_d^2$. But  due to the geometric constraint, we should also have $\sigma_a^2=\sigma_c^2$,  which  limits the significance of the result.

When the set $X$  can be equipped with structures such as order, and distance (this is the case if $X =   \mathbb{N} $ or  $X =   \mathbb{R} $, or Cartesian products thereof), other comparison metrics such as Wasserstein distance may be used. Indeed, consider the case $X =   \mathbb{R} $, and let us  denote $CDF_a$ the cumulative distribution of $a$, i.e.: $$CDF_a(y)=\int_{-\infty}^y a(x)dx $$ 
If $a:b::_{ari}c:d$, we have (by linearity of the integral):
$$\forall y \in \mathbb{R}, CDF_a(y) - CDF_b(y) = CDF_c(y) - CDF_d(y).$$
If we consider Wasserstein distance for $p=1$, because  we have \cite{GibbsSu,Ramdas}: $$W_1(a,b)=\int_{\mathbb{R}} |CDF_a(y)-CDF_b(y)|dy,$$
we will have (using the  $L1$ metric on $\mathbb{R}$) \ \
$W_1(a,b)=W_1(c,d).$ 

The developments outlined above would require a more precise, detailed and complete investigation. This is left for further research.

\section{From analogical proportions between profiles to  analogical proportions between probabilities}\label{compat}
In this section we discuss how an analogical proportion between four profiles may be associated with an   analogical proportion between probabilities\footnote{However note that positive values that add up to 1 are not necessarily probabilities. It may be also the coefficients of a weighted sum in multi-criteria aggregation. Thus,
we can adjust a set of coefficients $d$ with respect to three other sets of coefficients $a, b, c$ by solving
an analogical equation $a:b::c:x$ between the coefficient distributions which are associated with different profiles that form an analogical proportion with the profile corresponding to $d$, as in the following example:
\begin{itemize}
\item $a$: coefficients of the subjects in the science stream at school 1,
\item $b$: coefficients of the subjects in the literature stream at school 1,
\item $c$: coefficients of the subjects in the science stream at school 2,
\item $d$: coefficients of the subjects in the literature stream at school 2, 
\end{itemize}
}. We first recall how analogical proportions between profiles can be associated with analogical proportions between classes. We then consider a compatibility problem with probabilities where the existence of  analogical proportions between probabilities associated to profiles is a matter of adjustment. We finally discuss situations where one  may expect that an analogical proportion between profiles be coupled with an analogical proportion between probability distributions associated with the profiles. 

Probabilities and classification are naturally related in Bayes' theorem, which writes:
$$P(cl|x) P(x)= P(x|cl)P(cl)$$
where $x$ is a vector of attribute values describing an item, $cl$ is a class, and we assign to $x$, among the possible classes, the class $cl$ that maximizes $P(c|x)$ in Bayesian classification. It is worth noticing that this equality is a geometric proportion: 
$$P(x): P(x|cl) ::P(cl):P(cl|x)$$
which can be also written $\frac{P(x|cl)}{P(x)}=\frac{P(cl|x)}{P(cl)}$ or $ P(x): P(cl) :: P(x|cl):P(cl|x)$.

However, classification by analogical proportion does not follow the same mechanism, as we will now recall.
\subsection{Classification}\label{classification}
Analogical proportions may go well with classification. Defining two classes in the toy example of Table \ref{Example}:  
$cl_1$ { corresponding to} Age  $\leq 40  \mbox{ and  Education level }= high $,  $cl_2$ { corresponding to} Age $> 40    \mbox{ and Education level} = high$,
we obtain  Table \ref{ex1cl}, where we can see that the classification column offers a valid analogical proportion.  
\begin{table}[!ht] 
\caption{Analogical proportion with classes - 1 }\label{ex1cl}
\centering
$
\begin{array}{c||c|c|c|c|}
& \mbox{Education level } &   \mbox{ Gender}  &  \mbox{ Age} & \mbox{1st classification}  \\
\hline
a & \mbox{high} & M     &    \leq 40& cl_1 \\
\hline    
b & \mbox{high} &   F   &    \leq 40   & cl_1\\
\hline
c& \mbox{high} & M     &    > 40    & cl_2\\ 
\hline  
d & \mbox{high}  & F     &   > 40   & cl_2\\
\hline
\end{array}
$
\end{table}
Defining the two classes rather as $cl'_1$ { corresponding to} Gender $=M$  and   Education level$=high$ and $cl'_2$ { corresponding to} Gender $=F    \mbox{ and Education level}$ $=high$, we obtain Table  \ref{ex2cl}, where  the classification column exhibits another pattern of a valid analogical proportion. Note that in the cases of  Tables \ref{ex1cl} and \ref{ex2cl}  the classes are defined by means of simple logical expressions in terms of the given attributes.

\begin{table}[!ht]
\caption{Analogical proportion with classes - 2}
\centering
$
\begin{array}{c||c|c|c|c|}

& \mbox{Education level } &   \mbox{ Gender}  &  \mbox{ Age}  & \mbox{2nd classification}  \\
\hline
a  & \mbox{high} & M     &    \leq 40    & cl'_1  \\
\hline
b & \mbox{high} &   F   &    \leq 40    & cl'_2 \\
\hline
c & \mbox{high} & M     &    > 40    & cl'_1 \\
\hline
d & \mbox{high}& F     &   > 40    & cl'_2\\
\hline
\end{array}
$
\label{ex2cl}
\end{table}

Both Tables \ref{ex1cl} and \ref{ex2cl} illustrate a situation  such that if $a:b::c:d$, then it is also true for the classes, namely we have   $cl(a):cl(b)::c(c):cl(d)$. This corresponds to the following inference rule, which is the basis of the analogical classification procedure: the analogical proportion is preserved on the classes:
\begin{equation}\label{anainf}
   \frac{a:b::c:d}{cl(a):cl(b)::cl(c):cl(d)}
   \end{equation}
As such, it
is an {\it analogy conservation} rule. Analogical inference has been explicitly defined in the work of \cite{DavRus1987} as an {\it analogical jump} that from $R(x), R(y), S(x)$ infers $S(y)$ where $R, S$ denote properties and $x, y$ are items. It has been formally established that the inference rule (\ref{anainf}) is a special form of analogical jump    \cite{ijar/BounhasPR17}. 
This pattern of inference allows assigning a class to a new vector $d$, whose class is unknown, by solving the analogical equation
$cl(a):cl(b)::cl(c):x$ when $cl(a),cl(b), cl(c)$ are known (when a solution exists). 
In the case of multiple triplets $a, b, c$ leading to different solutions, we use the majority vote principle to
select a solution for the class to be predicted. {Using only selected triplets $a,b,c$  \cite{ijar/BounhasP23}, the above classification method predicting $cl(d)$ from $a, b, c, d$ and
$cl(a)$, $cl(b)$, $cl(c)$, has proven effective on benchmarks, performing at least as well as k-NN, C4.5, SVM, and Naive Bayes across 16 datasets from the UCI repository. See also \cite{ijar/BounhasP24}. }

It is clear that the inference pattern (\ref{anainf}) is not universally valid, and different triplets can lead to different conclusions. 
The reader is referred to \cite{CouHugPraRicIJCAI2017} and to \cite{CouceiroLMPRSUM2020} for the characterization of the classification functions for which the above inference pattern is always valid in the Boolean case and in the nominal case respectively. For the Boolean case,
it corresponds to functions that are affine (in terms of a sum which is the exclusive or and a product which is the conjunction). In \cite{CouHugPraRicIJCAI2018}, one can find an attempt at characterizing the error range in probabilistic terms when the inference rule (\ref{anainf}) is applied beyond its validity domain. The result has been recently improved in  \cite{aaai/CunhaLCB26}.

Until now, the form of reasoning allowed by pattern (\ref{anainf}) has been restricted to Boolean, nominal, and numerical values 
{for describing items}, and to Boolean or nominal values { for describing classes}. 
In the following, we discuss and explore potential extensions of this pattern to cases {in which} the consequence part involves probabilities.

\subsection{A compatibility problem}\label{compa} 
Let us start with a simple question. Since a set of attribute values may describe a single item as well as the profile of a subpopulation of items, one may wonder in this latter case if the probabilities of belonging to a profile 
form an analogical proportion when the profiles themselves form an analogical proportion. Let us take an example.

We use the example of Table \ref{Example} with men or women, under or over 40. 
When $a, b, c, d$  represent profiles within a given population of individuals, it is legitimate to consider
the proportion of individuals with profile $a, b, c$ or $d$, denoted $p(a), p(b), p(c), p(d)$, 
interpreted as approximations of the corresponding probabilities. 
For instance,
$p(a) = Prob (Gender=M, Age\leq 40 \ | \ Education   = high)$, etc. { as shown in Table \ref{exP1}.}
Note that in that particular case, we have  $p(a) + p(b) + p(c) + p(d) = 1$.
\begin{table}[!ht]
\caption{Analogical proportions with fixed probabilities}\label{exP1}
\centering
$
\begin{array}{c||c|c|c|c|}

&   \mbox{ Gender}  &  \mbox{ Age} & \mbox{Education} & \mbox{Probability}  \\
\hline
a  & M     &    \leq 40   & \mbox{high} & p(a) \\
\hline
b &   F   &    \leq 40   & \mbox{high} & p(b)\\
\hline
c  & M     &    > 40    & \mbox{high} & p(c)\\
\hline
d  & F     &   > 40   & \mbox{high} & p(d)\\
\hline
\end{array}\vspace{-0.5cm}
$
\label{ex1prob}
\end{table}
A more general situation, { introducing parameters $\eta$ and $\theta$,} is pictured in Table \ref{exP2}. Depending on the values of the parameters
{$\eta$ and $\theta$},
one may have $p(a) + p(b) + p(c) + p(d)$ less than, equal to or greater than 1.
\begin{table}[!ht]
\caption{Analogical proportions with parametric probabilities}\label{exP2}
\centering
$
\begin{array}{c||c|c|c|c|}
&   \mbox{ Gender}  &  \mbox{ Age} & \mbox{Education} & \mbox{Probability}  \\
\hline
a  & M     &    \leq \eta   & \mbox{high} & p(a) \\
\hline
b &   F   &    \leq \eta  & \mbox{high} & p(b)\\
\hline
c  & M     &    > \theta   & \mbox{high} & p(c)\\
\hline
d  & F     &   >\theta   & \mbox{high} & p(d)\\
\hline
\end{array}\vspace{-0.3cm}
$
\label{ex2prob}
\end{table}
Note that both in Tables \ref{exP1} and \ref{exP2}, $a:b::c:d$ holds { as nominal proportions for} each of the attribute column. Then a natural question is: under what conditions, one may also have $p(a) : p(b) ::_x p(c) : p(d)$ where the subscript  $x$ refers either to the arithmetical proportion or to the arithmetico-geometric proportion? 
 
Suppose we have an arithmetic proportion on the probabilities in Table \ref{exP1}, this means we have $p(a) - p(b) = p(c) - p(d)$. But since $p(a) + p(b) + p(c) + p(d) = 1$, this yields $p(b) + p(c) = p(a) + p(d) = 1/2$. 
For instance, suppose we have, among the people with high education, $60\%$ of people with Age $\leq$ 40 (with  $40\%$ of men and $20\%$ of women),  and $40\%$ with Age > 40 (with  $30\%$ of men and $10\%$ of women). This gives $0.4 : 0.2 :: 0.3 : 0.1$, which is an arithmetic proportion. 
In case we have the same number of people with Age $\leq$ 40 and Age $>$ 40, with $30\%$ of men in both cases, then we get $0.3 : 0.2 :: 0.3 : 0.2$, which is an arithmetico-geometric proportion. 
This  means that in the population under consideration, i.e., those with a high degree in education,
the  difference between the proportions of $a$ and $b$ is the same as between $c$ and $d$, 
here that the proportion between men and women with a “high” level of education does not depend on the Age group,  $\leq$ 40 or > 40. It expresses a form of equilibrium between the two age groups where there is the same difference between M and F in the context of high education. 

Adjusting the thresholds $\eta$, $\theta$ in Table
\ref{exP2} one may look for analogically balanced groups. But in general, it is clear that we may have $a : b :: c : d $ while $ p(a) : p(b) ::_x p(c) : p(d)$ may turn to be wrong. 

\subsection{Proportions between probability distributions associated to analogical profiles}\label{ana+distri}
In the example of the previous subsection, we observe that the probabilities associated with $a, b, c, d$ are not the probabilities of events referring to the values of some new attribute 
not used in the profile description,
but rather the probabilities that an item belongs to a given profile. 
This is why we have a ``vertical'' constraint linking the values of $p(a), p(b), p(c)$ and $p(d)$.

However, it would be more natural to expect an analogical proportion in case 
$a, b, c, d$ are no longer associated with a new ``property'' (the class in a classification problem), but with the probability of an event referring to a new property, or more generally to a probability distribution over the possible values of a new attribute (not used in the description of the profile). 
{Recall from Observation \ref{event} that} an analogical proportion between probability distributions leads to 
an analogical proportion  between the probabilities of any event computed from these distributions. Note also that, in this case, the constraints on the probability values become ``horizontal'' and express the normalization of the distributions.

This corresponds to the  following inference pattern  
where $p_a, p_b, p_c, p_d$ denote the 
probabilities of an event, or probability distributions. We use the arithmetic proportion in the pattern below since it is less  demanding than the arithmetico-geometric proportion. 

\begin{equation}\label{anadistri}
\frac{a:b::_{ari}c:d}{p_a: p_b::_{ari}p_c:p_d}
\end{equation}

\section{Illustration}\label{illustration}
In this section, our aim is to empirically check to what extent the inference rule (\ref{anadistri})  is valid, 
where $a, b, c, d$ represent profiles in a given population and $p_a, p_b, p_c, p_d$ are  probability distributions associated to these profiles.
\subsection{Datasets}
In order to approximate  probabilities with frequencies, we investigate relatively huge publicly available datasets:
\begin{enumerate}
\item The well-known \textbf{MovieLens 100K} dataset~\cite{harper2015movielens}, 
obtainable from \url{https://grouplens.org/datasets/movielens/100k}, 
comprises 100{,}000 ratings on a 1--5 scale from 943 users across 1{,}682 
movies, with each user having rated at least 20 movies. Features are divided 
into three groups: user demographics (age, gender, occupation, ZIP code), 
item metadata (movie title, release date, genre), and interaction data 
(timestamped ratings). The target variable is the \textbf{rating}, an integer 
ranging from 1 (poor) to 5 (excellent).
\item The \textbf{US Traffic Accidents} dataset~\cite{moosavi2019accidents}, 
obtainable from \url{https://www.kaggle.com/datasets/sobhan-moosavi/us-accidents}, 
records approximately 7.7 million accidents across 49 US states, with around 
47 features covering incident details, geolocation, weather conditions, and 
road characteristics. The target variable is \textbf{severity}, an integer 
from 1 to 4 indicating the impact on traffic flow, ranging from minimal 
disruption (1) to severe disruption (4).
\end{enumerate}
\subsection{Profiles}
In both datasets, we have defined what is the target variable we focus on (rating for MovieLens and severity for US Traffic Accidents).
To clarify the link with the vocabulary used in the previous sections, let us explain what are the profiles. We describe in detail for MovieLens but this is exactly the same philosophy for US Traffic Accidents.
In MovieLens, we only focus on user metadata and we first quantize the age range into 5 age groups: 18-24, 25-34, 35-44, 45-54, 55+.
For instance, focusing on profiles involving the 2 features $gender$ and $age\_group$, we have $2 \times 5 = 10$ candidate profiles.
But for $age\_group$ and $occupation$, we have 105 candidate profiles.
As a preliminary step, we discard all profiles that are not sufficiently represented in the dataset. To do that, we fix a minimum  threshold and call a profile {\it consistent} if it is satisfied by at least that many instances. Profiles below this threshold are too sparse to support reliable conclusions and are excluded from the analysis. As an example:
\begin{itemize}
    \item With 3 features $gender,age\_group,occupation$, the candidate profile $('F', '45-54', 'marketing')$ is represented by only 20 instances in the dataset,
    \item With 2 features, the candidate profile $('F', '35-44')$) is represented by 5061 instances.
\end{itemize}
Because restricted to these features, the dataset is categorical, the concept of analogical proportion between profiles is defined componentwise as described in Section \ref{background}. 
Then, looking for non trivial analogies involving only consistent profiles, we get as candidate analogy: 
$$('F', '35-44'):('F', '45-54')::('M', '35-44'):('M', '45-54')$$
where:
$$a=('F', '35-44'), b=('F', '45-54'), c=('M', '35-44'), d=('M', '45-54')$$
and where we have a population of 5061 for profile $a$, 3316 for profile $b$, 12093 for profile $c$ and 9182 for profile $d$.
Given a profile $a$, the conditional distribution $P(rating/a)$ is represented with an histogram with 5 bars because the candidate values of a rating of movies range from 1 to 5.
We provide Figures \ref{profile_movielens_1} and \ref{profile_movielens_2} as an illustration of histograms linked to two distinct profiles.
\begin{figure}[htbp]
    \centering
    \caption{Examples of Histograms for Movielens}
    \begin{subfigure}[b]{0.45\textwidth}
        \centering
        \includegraphics[width=\textwidth]{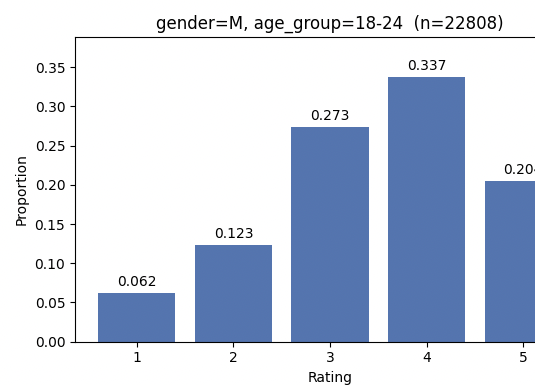}
        \caption{}
        \label{profile_movielens_1}
    \end{subfigure}
    \hfill
    \begin{subfigure}[b]{0.45\textwidth}
        \centering
        \includegraphics[width=\textwidth]{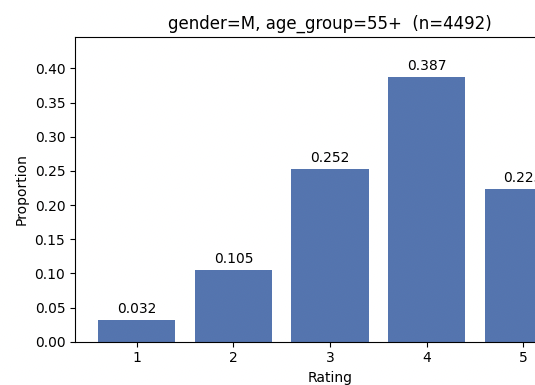}
        \caption{}
        \label{profile_movielens_2}
    \end{subfigure}
\label{profile_histo}
\end{figure}

\subsection{Results presentation}
In our experiments, each profile $a$ is associated with an histogram viewed as a vector in $\mathbb{R}^n$, where $n = 5$ for MovieLens and $n = 4$ for US Traffic Accidents. When four profiles $a, b, c, d$ form an analogical proportion, analogy conservation holds if and only if the four vectors form a parallelogram in $\mathbb{R}^n$. A natural measure of deviation from this conservation property is the analogical dissimilarity $\mathit{AD}$, defined in Subsection~\ref{ad-extension}.

Each subset of features defines a set of profiles, whose length is the number of features it comprises. Among all candidate profiles, only those meeting the consistency threshold are retained. From these consistent profiles, we extract all valid analogical proportions $a:b::c:d$ and we organise them into a box, where each row corresponds to one such proportion. The axes are as follows:\begin{itemize}
    \item The $y$-axis enumerates the quadruples $a, b, c, d$. 
    \item The $x$-axis displays the analogical dissimilarity $AD \in [0,1]$ between the histograms associated to $a, b, c, d$.
\end{itemize}
The lower this $x$-axis value, the better the analogy is preserved in the distribution space. To provide a fast visual understanding of the figures, each row is rendered as an horizontal bar coloured according to three levels:
green associated to $AD$ less than or equal to $0.05$, 
orange for $AD$ between $0.05$ and $0.10$ and red otherwise.
Ideally, we expect a majority of short green bars.
The title of each box identifies  the  feature subset used to build profiles.

\subsection{MovieLens results}
We have conducted experiments using profiles of length 2 and 3, as no consistent profiles of length 4 were available whatever the threshold. 
The results for profile of length 2 are gathered in Figures \ref{2-1000}, \ref{2-2000}, \ref{2-3000}, \ref{2-4000}.
\begin{figure}[htbp]
\centering
\caption{Profiles length: 2 - Consistency Threshold: 1000}
\includegraphics[width=1.0\linewidth]{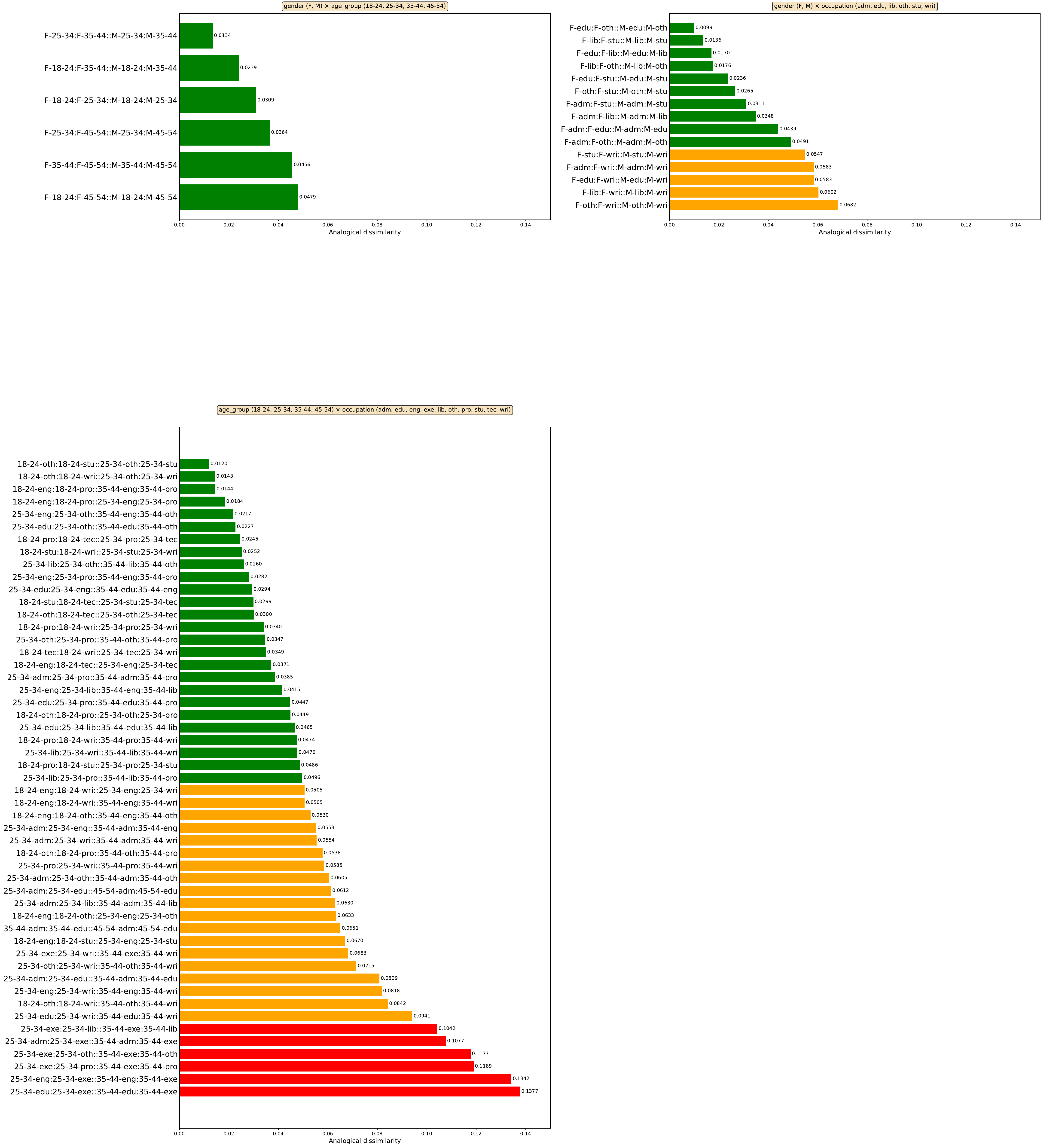}
\label{2-1000}
\end{figure}
\begin{figure}[htbp]
\centering
\caption{Profiles length: 2 - Consistency Threshold: 2000}
\includegraphics[width=1.0\linewidth]{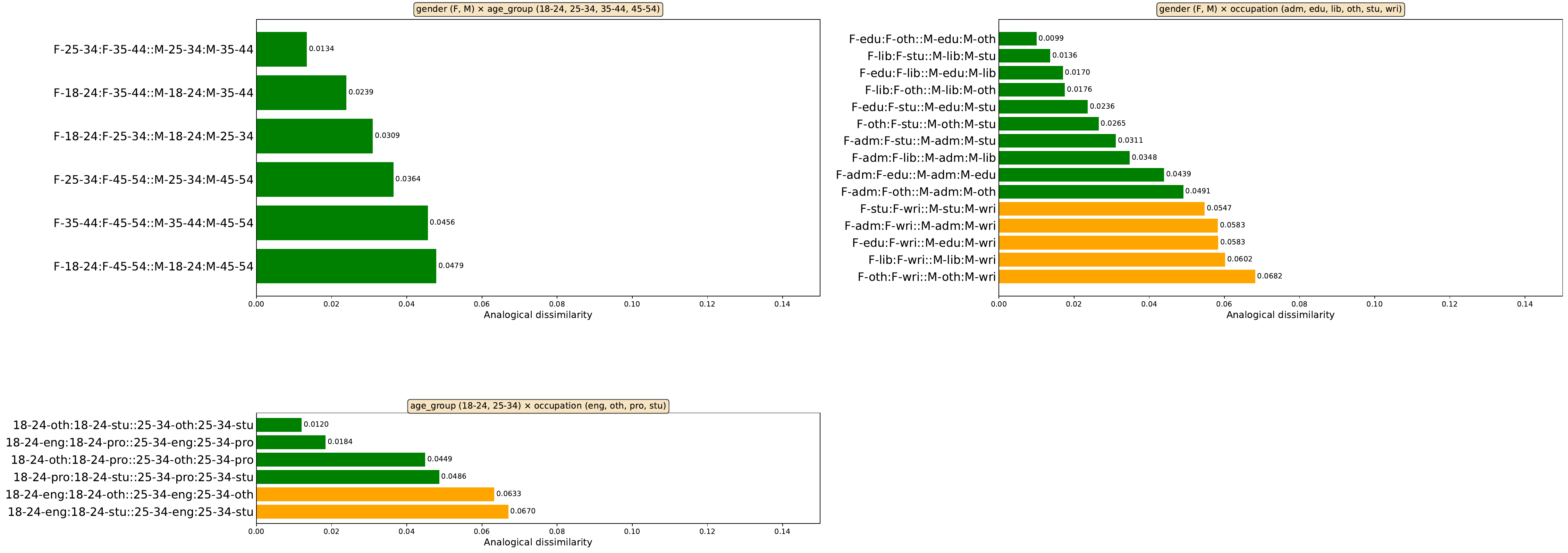}
\label{2-2000}
\end{figure}
\begin{figure}[htbp]
\centering
\caption{Profiles length: 2 - Consistency Threshold: 3000}
\includegraphics[width=1.0\linewidth]{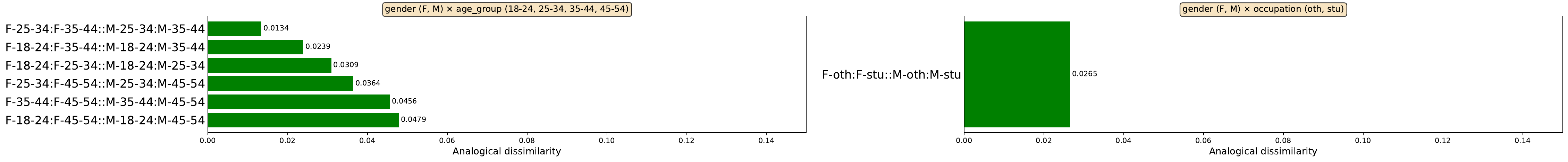}
\label{2-3000}
\end{figure}
\begin{figure}[htbp]
\centering
\caption{Profiles length: 2 - Consistency Threshold: 4000}
\includegraphics[width=1.0\linewidth]{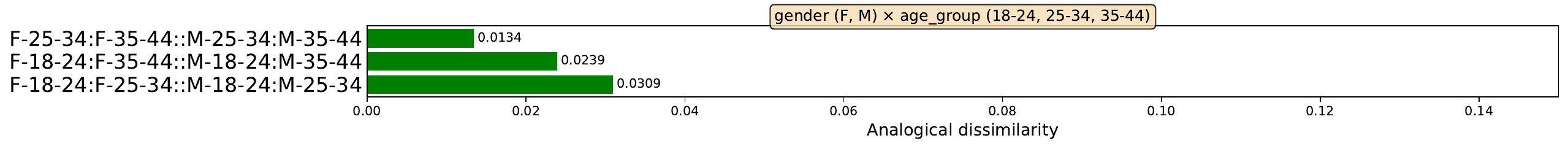}
\label{2-4000}
\end{figure}

The results for profile of length 3 are gathered in Figure \ref{fig:ml3-all}. It has to be noted that we do not have consistent profiles of length 3 when the consistency threshold is higher than 2000.
\begin{figure}[htbp]
    \centering
    \caption{Profiles length: 3}
    \begin{subfigure}[t]{0.49\linewidth}
        \centering
        \caption{Consistency threshold $= 1000$}
        \includegraphics[width=\linewidth]{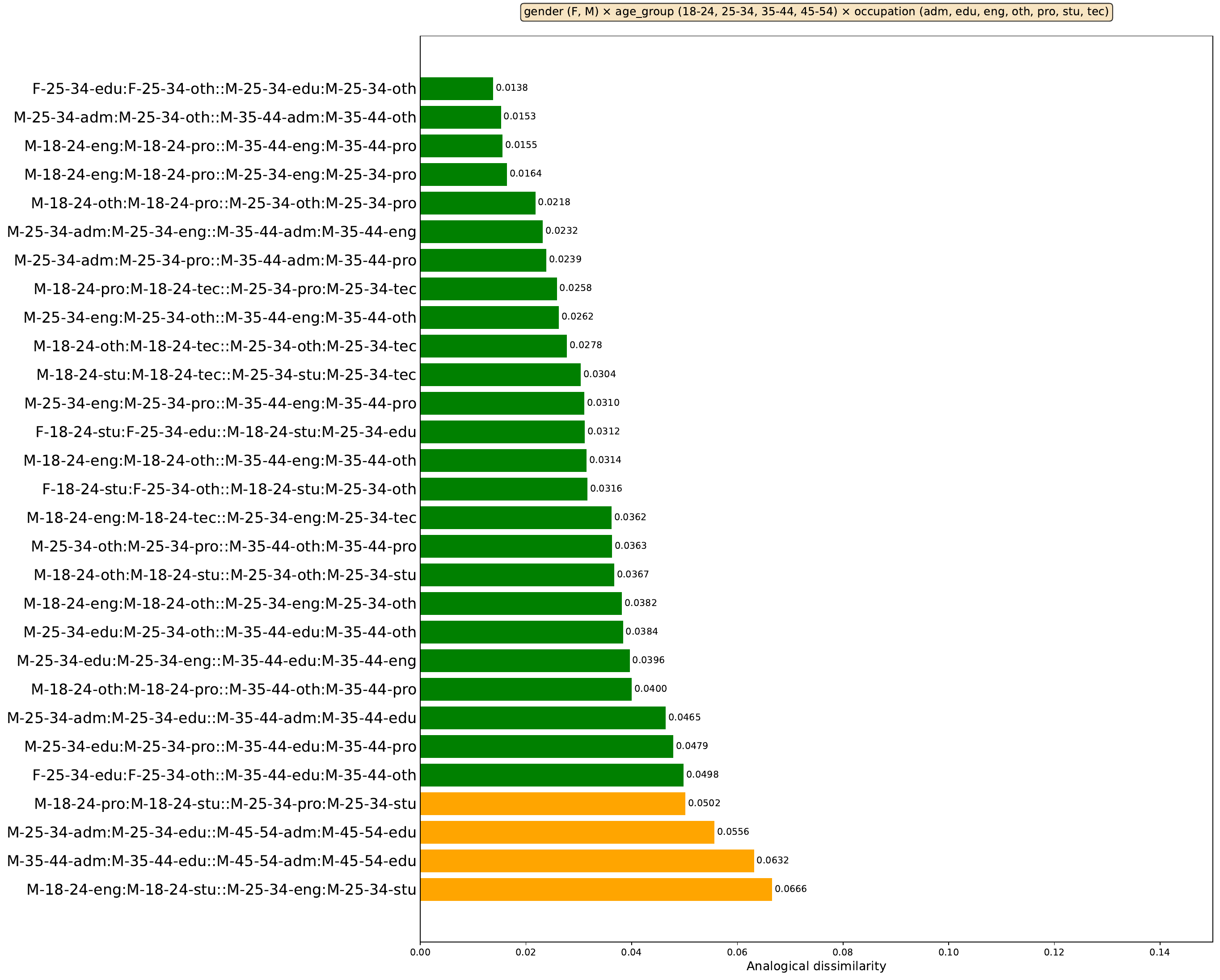}  
        \label{fig:ml3-thr1000}
    \end{subfigure}
    \hfill
    \begin{subfigure}[t]{0.49\linewidth}
        \centering
        \caption{Consistency threshold $= 2000$}
        \includegraphics[width=\linewidth]{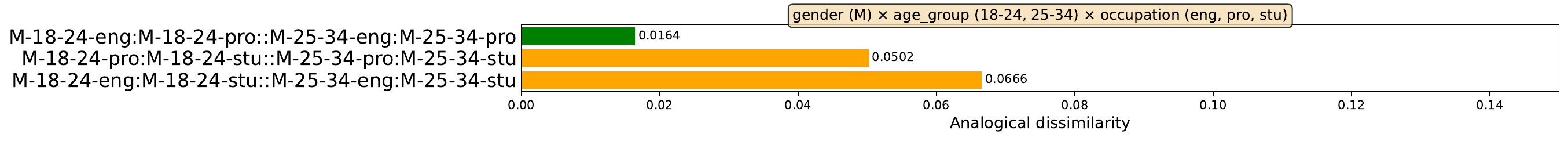}
        \label{fig:ml3-thr2000}
    \end{subfigure}
    \label{fig:ml3-all}
\end{figure}
We may note that analogical dissimilarity is always smaller than 0.2 that is why we have limited the x-axis to 0.2.
\newpage
\subsection{US Traffic Accidents results}
We proceed similarly as with MovieLens with some differences:
\begin{itemize}
    \item The original 47 features are grouped under four categorical variables $timezone,day\_night, weather\_group, wind\_group$ that we use as the profile features in our experiments. For instance, $weather\_condition$ has high cardinality that we group into $weather\_group$ with only four possible values $Clear, Rain, Snow, Fog$.
    \item Due to the huge size of the dataset, we sample only 300000 instances among 7.7 millions.
    \item We increase the consistency threshold.
\end{itemize}
In this context, the variable of interest is accident severity, which takes integer values from 1 to 4. For a given profile, the associated distribution over severity levels is represented by an histogram of 4 bars, as illustrated in Figures~\ref{histo-usa-1} and~\ref{histo-usa-2}.:
\begin{figure}[htbp]
    \centering
    \caption{Examples of Histograms for US Traffic Accidents}
    \begin{subfigure}[b]{0.45\textwidth}
        \centering
        \caption{}
        \includegraphics[width=\linewidth]{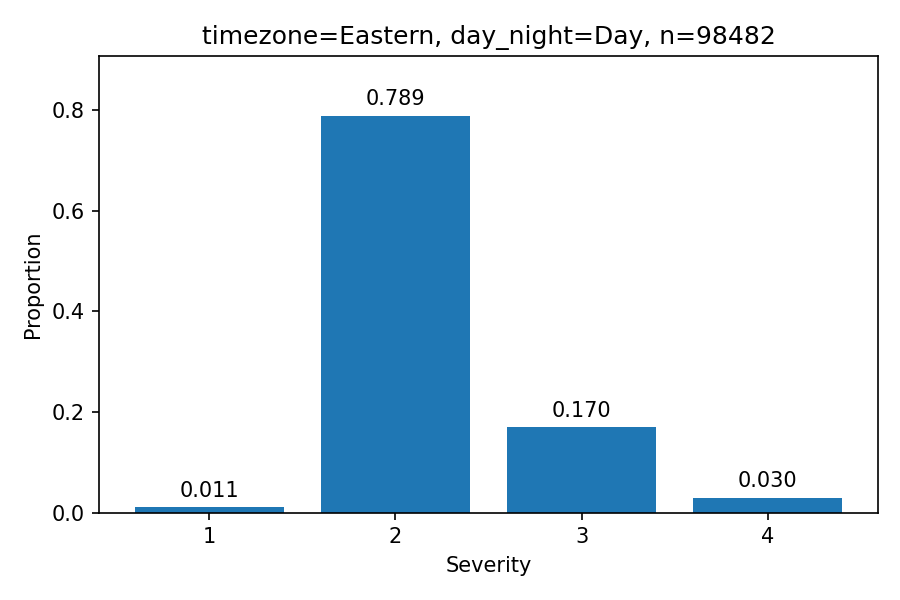}
        \label{histo-usa-1}
    \end{subfigure}
    \hfill
    \begin{subfigure}[b]{0.45\textwidth}
        \centering
        \caption{}
        \includegraphics[width=\linewidth]{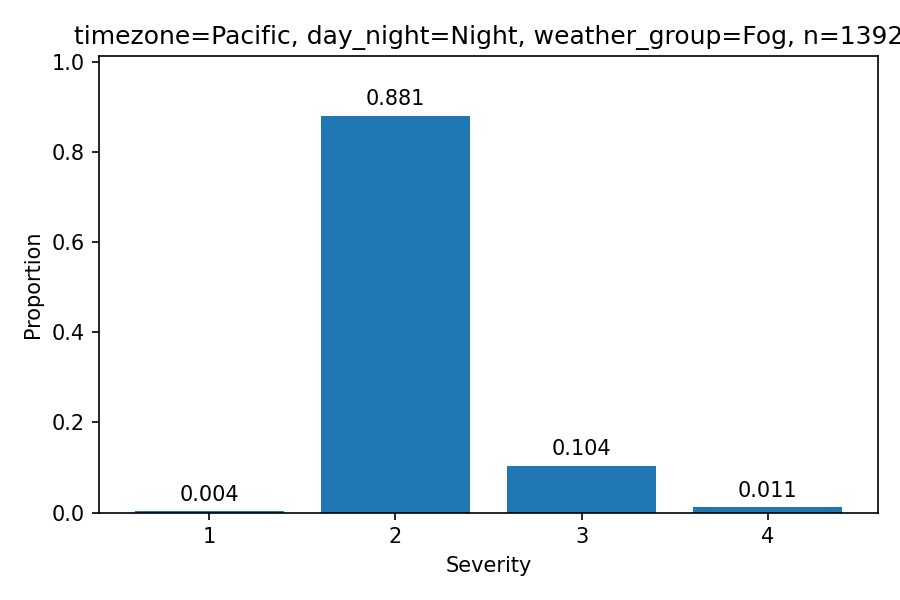}
        \label{histo-usa-2}
    \end{subfigure}
\label{histo-usa}
\end{figure}

Results for profile of length 2:
\begin{figure}[htbp]
\centering
\caption{Profiles length: 2 - Consistency Threshold: 5000}
\includegraphics[width=1.0\linewidth]{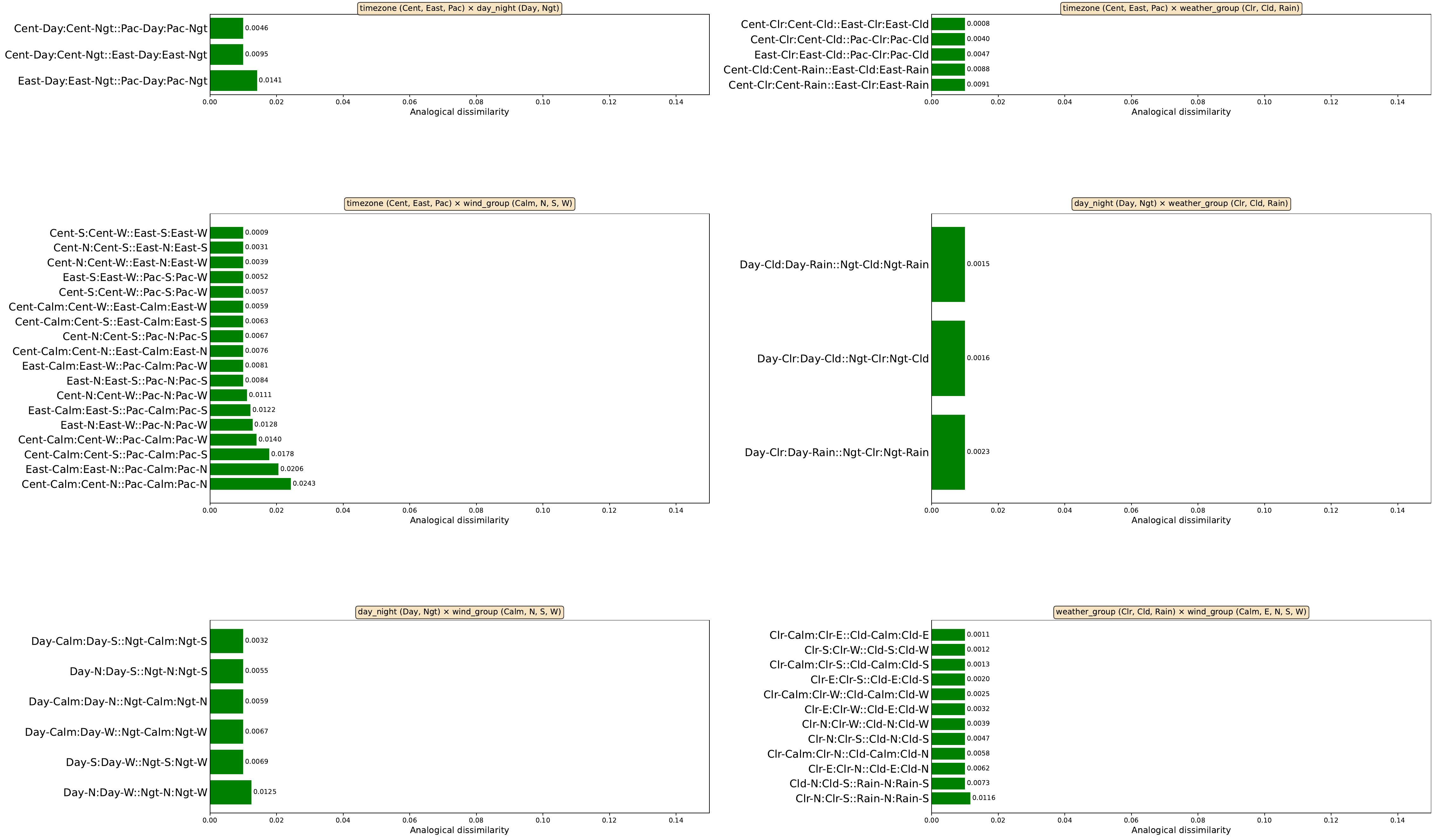}
\label{usa-2-5000}
\end{figure}

\begin{figure}[htbp]
\centering
\caption{Profiles length: 2 - Consistency Threshold: 10000}
\includegraphics[width=1.0\linewidth]{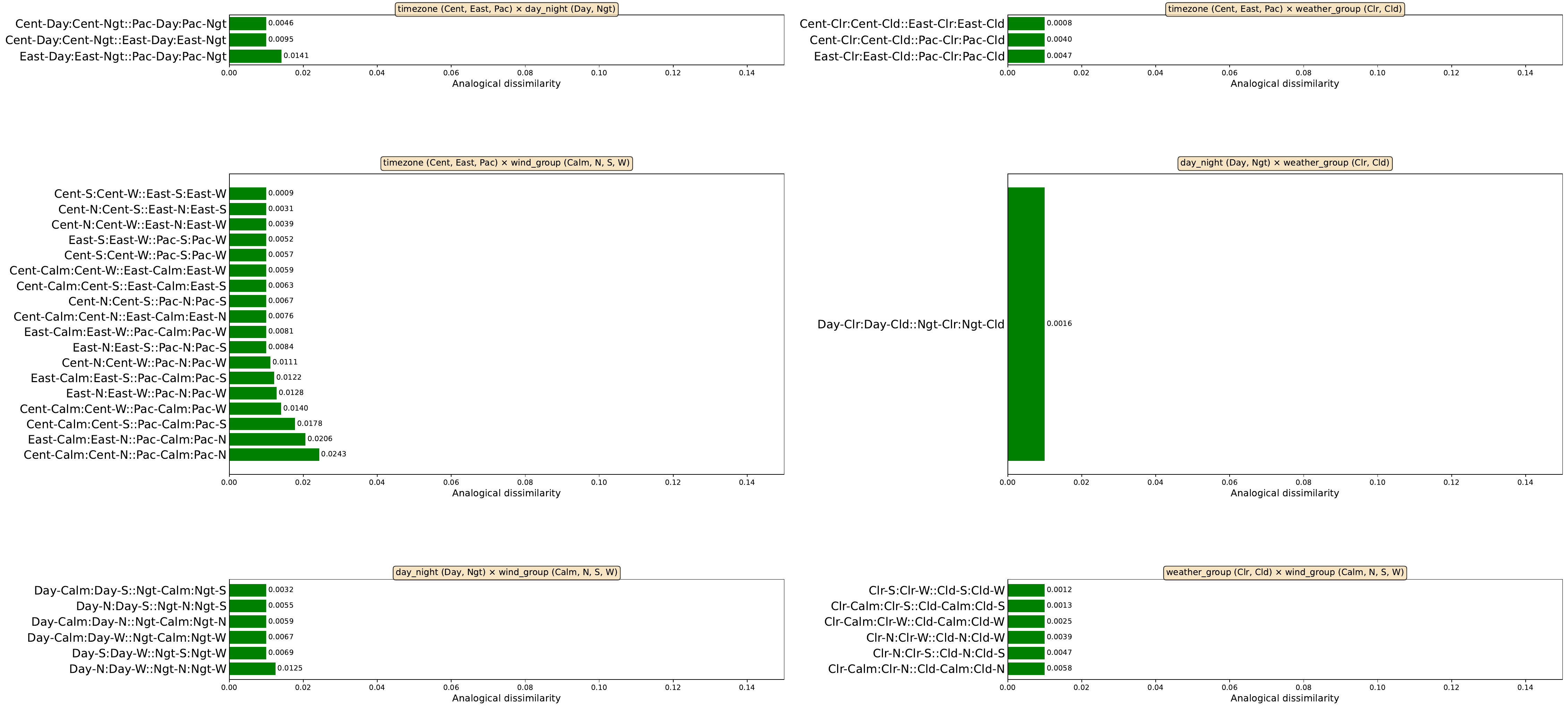}
\label{usa-2-10000}
\end{figure}

When moving to profile of length 3, the threshold of 5000 provides a lot of valid analogical proportions so we move directly to a threshold of 10000 instances.

\begin{figure}[htbp]
\centering
\caption{Profiles length: 3 - Consistency Threshold: 10000}
\includegraphics[width=1.0\linewidth]{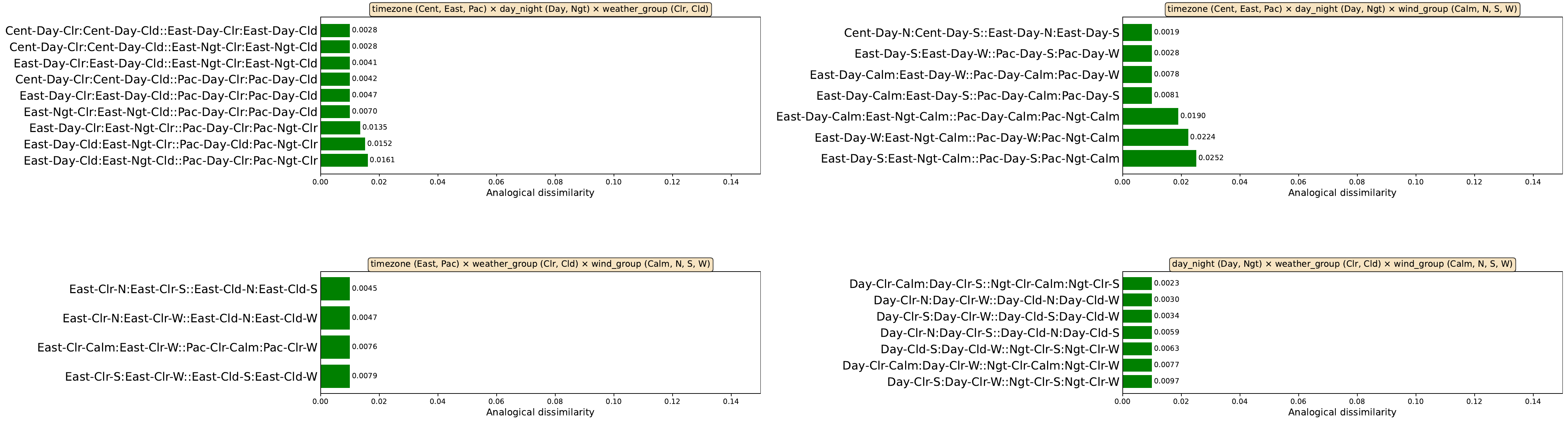}
\label{3-10000}
\end{figure}
\newpage
\subsection{Comments}
The inference rule (\ref{anadistri}) corresponds here to the expression of simple questions.
For instance in the case of MovieLens, we could ask: \\
{\it"Does the relationship between male and female ratings in age group $25-34$ mirror the same relationship in age group $35-44$, across rating distributions?"}\\
We test the validity of the inference rule for the following analogical proportion between profiles:
$$('M', '25-34') : ('F', '25-34') :: ('M', '35-44') : ('F','35-44')$$
The four profiles are well above the threshold of 4000 since we have ($n$ denotes the number of instances of the corresponding profiles):
\begin{itemize}
    \item $('M', '25-34')$: n=25685 and $('F', '25-34')$: n=9109
  \item $('M', '35-44')$: n=12093 and $('F', '35-44')$: n=5061
\end{itemize}
Computing the analogical dissimilarity $AD$ for there associated distribution gives $0.0134$  which agrees with the hypothesis of the validity of the inference rule (\ref{anadistri}).
But for $$('25-34', 'educator'):('25-34', 'executive')::('35-44', 'educator'):('35-44', 'executive'),$$ we find
$AD=0.1377$, indicating that there is a distortion of more than 13\% and we cannot consider this contributes to the  validation of the inference rule.

In the case of US Traffic Accidents, a question such as: \\
{\it "Does the relationship between daytime and nighttime accidents in timezone $Eastern$ mirror the same relationship in timezone $Pacific$, across severity distributions?"} is also directly translated as questioning the validity of the conservation rule for an analogical proportion:
$$('Eastern', 'Day') : ('Eastern', 'Night') :: ('Pacific', 'Day') : ('Pacific', 'Night')$$
with corresponding numbers of instances:
\begin{itemize}
    \item $('Eastern', 'Day')$: n=98482 and $('Eastern', 'Night')$: n=40553
    \item  $('Pacific', 'Day')$: n=52275 and $('Pacific', 'Night')$: n=28349
\end{itemize}
and $AD=0.0141$

As a conclusion, in view of our green bars where $AD < 0.05$, our investigations on the two datasets show that, when dealing with a sufficiently large number of instances, the analogical proportion between profiles are associated with an arithmetic proportion between the distributions.

\subsection{Investigation with the arithmetico-geometric proportion}
In the previous subsections, we have tested the validity of the inference rule (\ref{anadistri}) using a definition of the analogical proportion between distributions based on the arithmetic proportion. We may wonder if the inference rule is still valid for the arithmetico-geometric definition. For testing it, we need to estimate how far we are from a geometric proportion (for the same cases where have tested the arithmetic proportion). 
Because of the small numbers we deal with, in practice this means working in log-space where:
$$AD_{geo}(a,b,c,d) = \frac{1}{2n}\Sigma_i | log(\frac{a_i}{b_i}) - log(\frac{c_i}{d_i}) | = \frac{1}{2n}\Sigma_i | log(a_i) - log(b_i) - log(c_i) + log(d_i)|$$
We have first investigated  MovieLens with profiles of length 2 and 3, and threshold of 2000 and 4000: results are in Figures
\ref{movielens-geo-2-2000}, \ref{movielens-geo-2-4000} and \ref{movielens-geo-3-2000}.
\begin{figure}
    \centering
    \caption{MovieLens - Profile length: 2 - threshold: 2000}
    \includegraphics[width=1.0\linewidth]{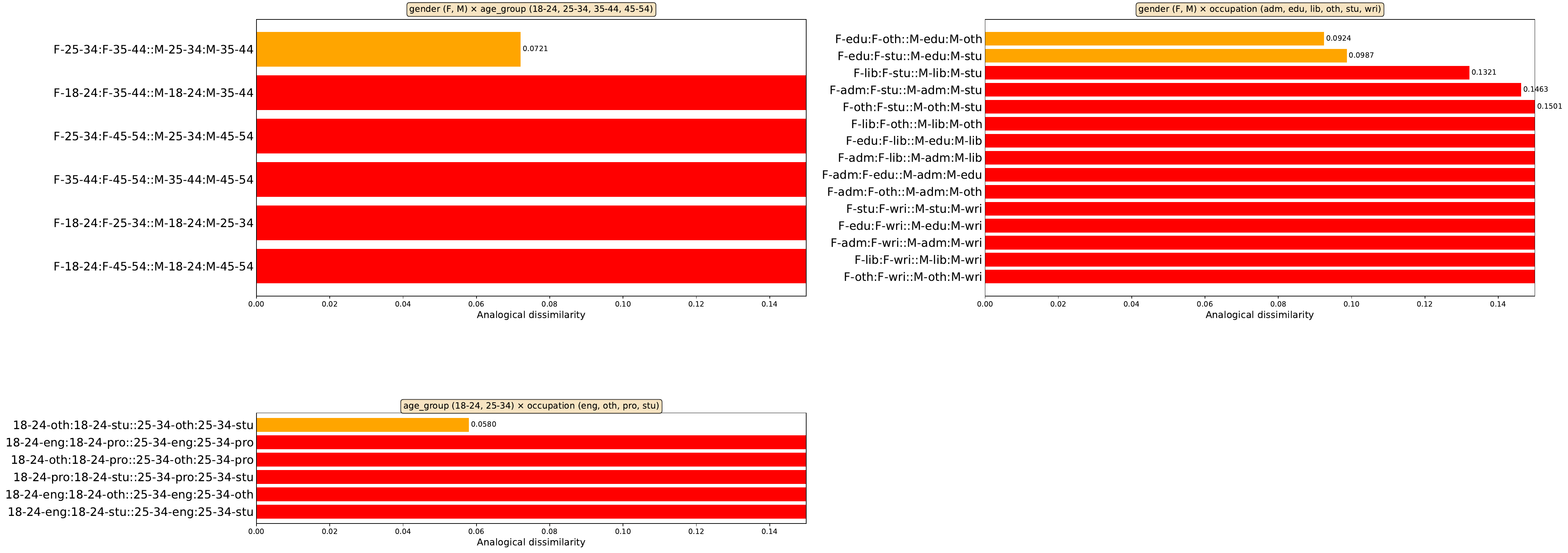}
    \label{movielens-geo-2-2000}
\end{figure}
\begin{figure}
    \centering
    \caption{MovieLens - Profile length: 2 - threshold: 4000}
    \includegraphics[width=1.0\linewidth]{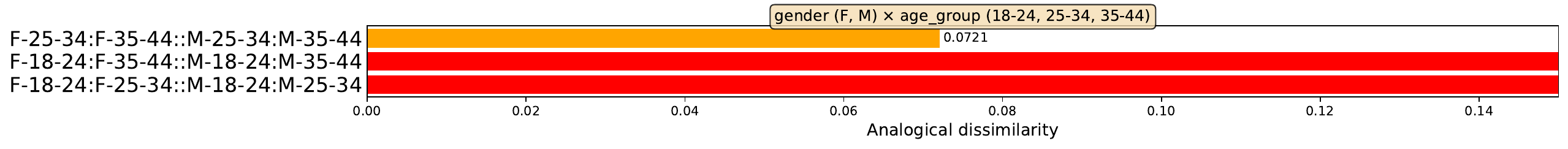}
    \label{movielens-geo-2-4000}
\end{figure}
\begin{figure}
    \centering
    \caption{MovieLens - Profile length: 3 - threshold: 2000}
    \includegraphics[width=1.0\linewidth]{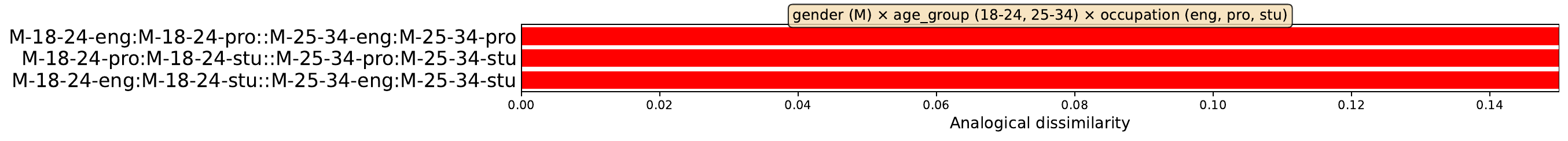}
    \label{movielens-geo-3-2000}
\end{figure}

Then we have investigated US Traffic Accident with profile of length 3 and threshold 10000. Results are in Figure \ref{usa-geo-3-10000}.
\begin{figure}
    \centering
    \caption{US Traffic Accidents - Profile length: 3 - threshold: 10000}
    \includegraphics[width=1.0\linewidth]{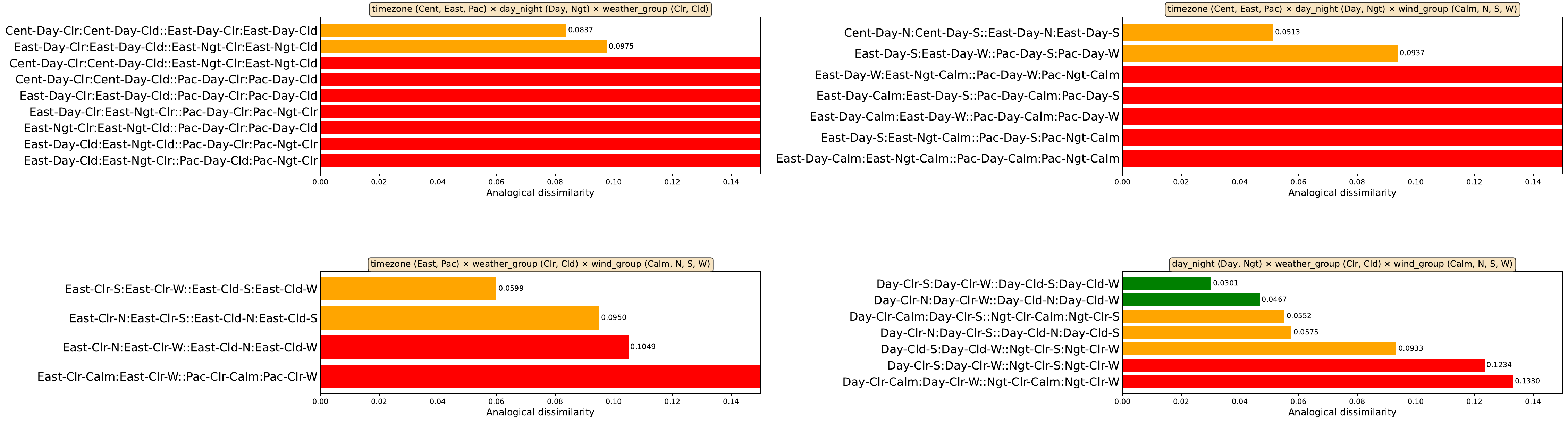}
    \label{usa-geo-3-10000}
\end{figure}

These preliminary experiences suggest that the arithmetico-geometric proportion between probability distributions should be quite rare, which is not a surprise since it is very  demanding.
\subsection{Final discussion}
The results of the experiments with the two data sets indicate that given four profiles forming an analogical proportion, one may expect that distributions of probability associated with them also form an analogical proportion. Obviously there are also cases where the results significantly deviate from this expectation. On the one hand, it is clear that what the results show is troublesome, while on the other hand we would need to understand why it is so. We are at the same stage as we were back when the first analogical classifiers \cite{MicBayDelJAIR2008} achieved those surprising results, before explanations began to emerge, as recalled in subsection  \ref{classification}. More theoretical research is needed to shed light on this issue. More experiments would also be  needed, in particular with profiles including numerical attributes.

Following a line of thought quite similar to case-based reasoning  \cite{sp/20/FuchsLMMNPR20}, transfer learning can be seen as a kind of analogical reasoning performed at a meta-level (see for example \cite{aaai/WangY11}), since it involves taking advantage of what has been learned in a source domain in order to improve the learning process in a target domain related to the source domain. Thus, 
one might wonder whether analogical proportions between probability distributions, in the sense of Definitions  either \ref{ProbAri}, or    \ref{sophi}, could allow for more sophisticated and controlled transfers, using Propositions \ref{ProposAri}, or \ref{ProposNomi} and \ref{exist}.

Since analogical proportions enable the computation of
a missing value from three known ones,
they offer a simple method of estimating a probability distribution associated with a new profile from three other distributions without having to analyze the corresponding population.
\section{Conclusion}
This article is based on the material developed in the two first papers \cite{PraRicEcsqaru2025,miwai/PradeR25} that  explored analogical proportions between probability distributions.
The present paper reorganizes and expands on their contents,
providing a systematic study of the subject, rich in technical findings, 
while presenting experiments that demonstrate the validity of this idea. 
In particular, it has been shown that two definitions of analogical proportions between probability distributions make sense in the finite case: one based only on the  arithmetic proportion  that preserves the total variation distance between the distributions forming the two pairs of the analogical proportion.  
A definition combining arithmetic and geometric proportions is also possible.
This latter definition ensures, the invariance of the Kullback-Leibler divergence between the distributions forming the two pairs in the analogical proportion. Because we define analogical proportion between finite probability distributions in a pointwise manner, the extension to the continuous case requires careful study left for further research  
 (beyond some examples given in Appendix 2). 
 
In this paper, we have examined the analogical relationship that may 
hold between profile  descriptions and  their associated probability distributions. More precisely, the central question is: 
when an analogical proportion holds between four distinct profiles in a population,  
{under what conditions} does a corresponding proportion hold between their associated probabilities? 
Further research is needed to clarify the precise nature of this relationship and to determine the validity and usefulness of such an analogical relationship. 
Moreover, concrete applications remain to be developed.\\

\noindent{\bf Acknowledgments}. 
The authors thank Giuseppe Sanfilippo for posing the question “Do analogical proportions apply to probabilities?” after the presentation of their article \cite{sum/PradeR24}. This question leads them into the exploratory work presented here, and this article hopes to demonstrate that the question was fruitful. This research was supported by the ANR project “Analogies: from Theory to Tools and Applications” (AT2TA), ANR-22-CE23-0023.
\appendix
\section*{Annex: Examples of analogical proportion between continuous distributions} 
We start from a parametric function $S_{x_0, a} $, 
defined on $[0, + \infty)$ and whose graph is displayed in plain line in the next figure, where point $A$ has coordinates $(x_0, y_0)$. 
\begin{figure}[!ht]
    \centering
\begin{tikzpicture}[scale=1.1]
\coordinate [label=right:$ $] (B) at (2.33cm,-1.0cm);
\coordinate [label=above:$A(x_0$] (G) at (2.1cm, -0.13cm);
\coordinate [label=above:${,}y_0)$] (G1) at (2.7cm, -0.13cm);
\coordinate [label=above:$1$] (C) at (1cm,1.0cm);
\coordinate [label=above:$ $] (D) at (1cm,0.1cm);
\coordinate [label=below:$0$] (H) at (1cm, -1cm);
\coordinate [label=below:$a$] (E) at (5cm, -1cm);
\coordinate [label=below:$ $] (F) at (5.7cm, -1cm);\draw   (C) -- node[sloped,above,] { } (G) ;
\coordinate [label=above:$T_1$] (T1) at (1.3cm,-0.15cm);
\coordinate [label=above:$T_2$] (T2) at (2.1cm,-0.7cm);
\draw[dotted] (C) -- node[sloped,above] { } (C|-A) ;
\draw [dotted] (B) --  node [sloped, above] ( ) { } (C|-A);
\draw [dotted]  (E) -- (B);
\draw    (G) -- (E);
\draw [dashed] (G) -- (H);
\draw  (E) -- (F);
\end{tikzpicture}
\end{figure}

We can then build an example in the following way {(where $a,b,c,d$ are now probability density functions):}

$a(x) = p(x)$ on $(- \infty, + \infty)$

$b(x) = p(x)$ on $(- \infty, 0]$ ; $b(x) = q(x)$ on $[0, + \infty)$

$c(x) = q(x)$ on $(- \infty, 0]$; {$c(x) = p(x)$ on $[0, + \infty)$ }

$d(x) = q(x)$ on $(- \infty, + \infty)$\\
where:

$p(x) = S_{x_0, a}(x) $ on $[0, + \infty)$, $p(x) =S_{x_0, a}(- x)$ on $(- \infty, 0]$.

$q(x)$ is similarly defined with parameters $x'_0$ and $a'$.\\
Note that $p(0) = q(0) = 1$, then $a, b, c, d$ are continuous.
The area under $S_{x_0, a}$ is equal to $area(T_1) + area(T_2) = x_0/2 + ay_0/2$.
To ensure the normalization of the distributions, this area should be equal to $1/2$, we have then to add the constraint that:
$(x_0 + a y_0) = 1$ and $(x'_0 + a' y'_0) = 1$. 
In that context, we can easily check that \quad\quad\quad\quad $a : b ::_{arigeo} c : d $ and  $KL(a, b) =   KL(c, d)$.  

Starting with the above example, and using parametric functions  $S_{h, x_0, a} $   and $S_{h, x'_0, a'} $ similar to the previous ones, but such that $S_{h, x_0, a}(0)=S_{h, x'_0, a'}(0) =h$, one can build  a more complex example:

Let $p_1 < p_2$. Let $\lambda <1/2$.

$a(x) = p(x)$ on $(- \infty, p_1]$; $a(x) = r(x)$ on $[p_1, p_2]$; $a(x) = p(x)$ on  $ [p_2, + \infty)$

$b(x) = p(x)$ on $(- \infty, p_1]$; $b(x) = r(x)$ on $[p_1, p_2]$; $b(x) = q(x)$ on $[p_{{2},} + \infty)$ 

$c(x) = q(x)$ on $(- \infty, p_1]$; {$c(x) = r(x)$ on $[p_1, p_2]$; $c(x) = p(x)$ on $[p_2, + \infty)$ }   

$d(x) = q(x)$ on $(- \infty, p_1]$; $d(x) = r(x)$ on $[p_1, p_2]$; $d(x) = q(x)$ on  $ [p_2, + \infty)$

where $p(x) = S_{h,x_0, a}(x - p_1) $ on $[p_1, + \infty), S_{h,x_0, a}(p_1 - x)$ on $(- \infty, p_1]$ and

$q(x) = S_{h,x'_0, a'}(x - p_2) $ on $[p_2, + \infty), S_{h,x'_0, a'}(p_2 - x)$ on $(- \infty, p_2]$,
 
 $r(x)$ is such that $r(p_1)=r(p_2)= h$ { (ensuring continuity of the 4 distributions)} and $ \int_{p_1}^{p_2} r(x)dx = 2\lambda  $ ($r$ may be taken, e.g., as uniform; in this case $h(p_2-p_1)=  2\lambda$),\\
with the constraints  
$(hx_0 + a y_0) = 1 - 2\lambda $, $(hx'_0 + a' y'_0) = 1  - 2\lambda$, 
(for ensuring normalization of the 4 distributions).
\bibliographystyle{elsarticle-num}
\bibliography{biblio}
\end{document}